\documentclass{article}

\usepackage[preprint]{neurips_2026}

\usepackage[utf8]{inputenc}
\usepackage[T1]{fontenc}
\usepackage{hyperref}
\usepackage{url}
\usepackage{booktabs}
\usepackage{amsfonts}
\usepackage{nicefrac}
\usepackage{microtype}
\usepackage{xcolor}

\usepackage{macros}
\usepackage{amsmath}
\usepackage{amsthm}
\usepackage{tikz}
\usepackage{graphicx}
\usepackage{subcaption}

\usepackage{cleveref}
\usepackage{placeins}

\usepackage{tabularx}
\usepackage{array}

\makeatletter
\@ifundefined{theorem}{\newtheorem{theorem}{Theorem}}{}
\@ifundefined{assumption}{\newtheorem{assumption}{Assumption}}{}
\makeatother

\title{Hierarchical Concept Geometry in Language Models Emerges from Word Co-occurrence}

\author{%
  Andres Nava \\
  Johns Hopkins University \\
  \texttt{anava1@jh.edu} \\
  \And
  Matthieu Wyart \\
  EPFL \& Johns Hopkins University \\
  \texttt{mwyart1@jh.edu} \\
}

\begin{document}

\maketitle

\begin{abstract}
We propose a distributional theory of how hypernymy---the ``is-a'' relation between general and specific concepts---is encoded geometrically in language representations. Starting from the empirically verified assumption that words closer on the WordNet hypernym graph co-occur more often, we characterize theoretically the spectrum of the resulting embedding Gram matrix of word2vec embeddings. Under mild positivity and decay conditions on the co-occurrence kernel, we prove that the leading eigenvectors first separate broad taxonomic branches and then progressively finer sub-branches, producing a \emph{hierarchical splitting geometry} with a coarse-to-fine spectral organization that mirrors the tree. We confirm these predictions in word2vec embeddings across many sampled WordNet subtrees, and show that the same signature extends strikingly well to Gemma 2B unembeddings. Our results indicate that hierarchical concept geometry in LLMs need not reflect a hierarchy-specific functional mechanism, but emerges from the spectral structure of pairwise word statistics.
\end{abstract}

\begin{figure}[!htbp]
  \centering
  \includegraphics[width=0.95\linewidth]{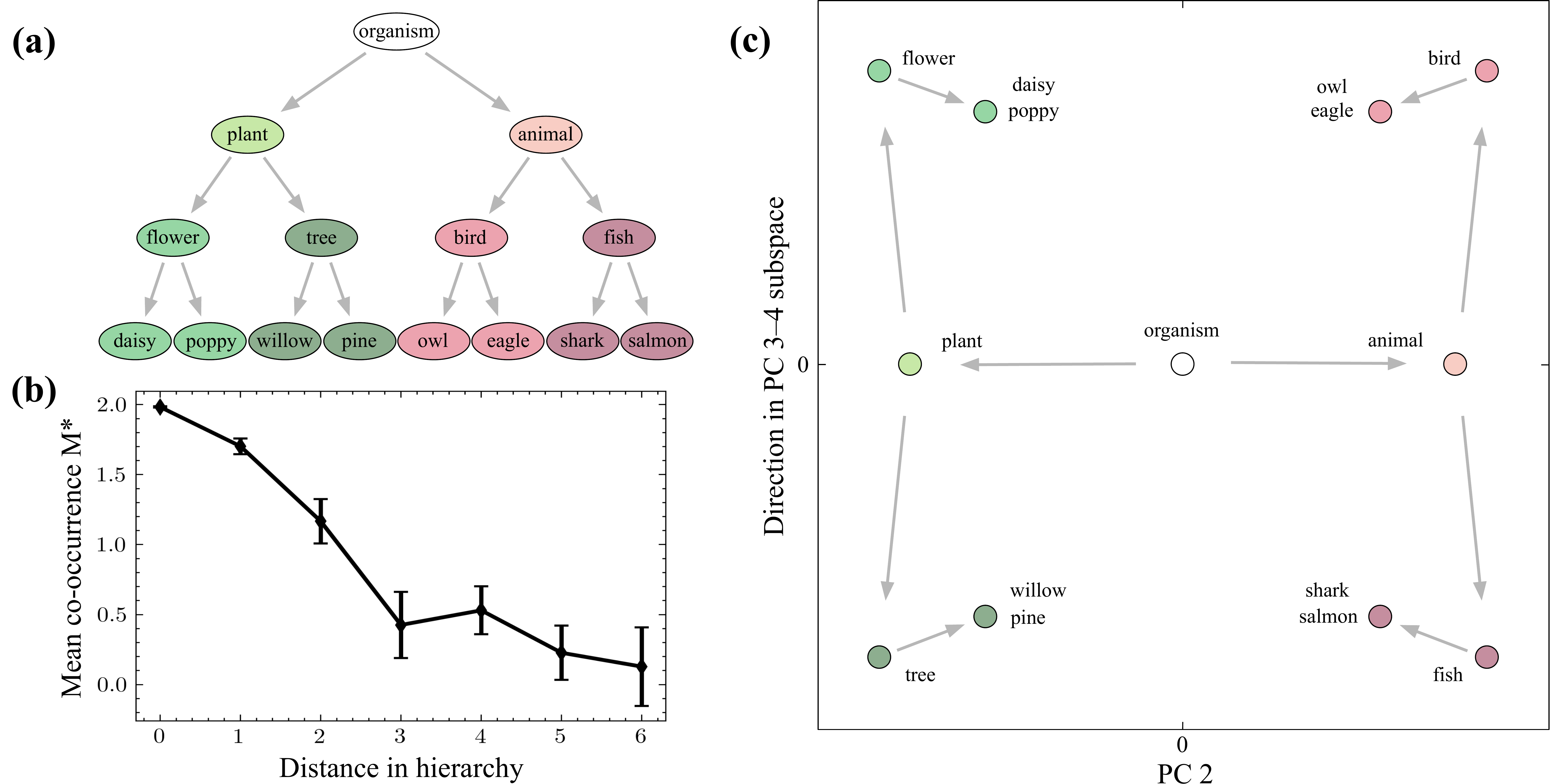}
  \caption{
      (a) WordNet hierarchy for a taxonomy of organisms.
      (b) Mean co-occurrence statistic $M^\star$ as a function of
      hierarchy distance in the organism taxonomy. Error bars indicate one
      standard error; nearby concepts co-occur more often.
          (c) Projections of taxonomy nodes onto principal components (PCs)
      predicted by the theory. PC 2 splits the plant and animal subtrees,
      while the degenerate PC 3--4 subspace contains directions that split
      flower vs.\ tree and bird vs.\ fish. We show one selected direction
      in this subspace.
    }
  \label{fig:figure1}
\end{figure}

\FloatBarrier

\section{Introduction}

While the success of Large Language Models (LLMs) in learning languages and reasoning from examples is mesmerizing, understanding the reasons for this success remains a challenge. A common belief is that LLMs build a representation of language that allows for simple manipulation. Recent work has shown that semantic variables often appear geometrically in representation spaces, forming directions, subspaces, loops, prisms, and other low-dimensional structures \cite{engels2025not, merullo2024language, modell2025origins, gurnee2024language, gurnee2025when}, some of which were recently found in the brain \cite{zhu2026geometric}. Such organization indeed allows one to manipulate concepts with simple arithmetic operations, such as vector addition or rotation. In this paper, we ask what geometric structure is induced by hypernymy, the ``is-a'' relation between more specific and more general concepts. It is one of the most basic organizing principles of meaning: an \emph{owl} is a \emph{bird}, a \emph{bird} is an \emph{animal}, and an \emph{animal} is an \emph{organism}. We take the WordNet hypernym graph~\citep{miller1995wordnet} as our operational definition of semantic hierarchy, with hypernyms corresponding to more general concepts and hyponyms to more specific concepts. This relation is exemplified on a taxonomy of organisms in \cref{fig:figure1}(a).

Specific algorithms have been introduced to bias representations to be hierarchical \cite{nickel_poincare_2017,nguyen2026hierarchicalconceptembedding}, however, LLMs already spontaneously form hierarchical representations in some form. Indeed,
linear transformations can predict ancestor or hypernym representations from descendant or hyponym representations in both skip-gram embeddings
\cite{fu_2014} and LLM hidden activations
\cite{sakata2026linearrepresentationshierarchicalconcepts},
while  linear transformations of LLM activations can recover WordNet hypernym-graph distances
\cite{orhan2026emergencephonemicsyntacticsemantic}. Park et al.~\cite{park2025geometrycategoricalhierarchicalconcepts} give a functional interpretation of hierarchical geometry in LLMs. Their starting point is the linear representation hypothesis: the idea that high-level semantic variables are encoded by linear directions or concept vectors in the model's representation space. In their framework, a concept vector is meaningful because of the operations it supports: moving in that direction should increase the model's probability of expressing the target concept, while leaving unrelated semantic variables unchanged. Thus, steering between \emph{bird} and \emph{fish} should change which kind of animal is represented, but should not change the probability of being an animal at all. For semantic hierarchies, Park et al.'s postulate implies a particularly elegant geometry. A descendant concept such as \emph{bird} should preserve the broader \emph{animal} component while adding an orthogonal refinement that distinguishes birds from other animals. In such a geometry, inner products with concept vectors act as clean membership readouts: an animal direction detects animalhood, a bird direction detects birdhood, and progressively finer directions detect progressively more specific classes.

Our contribution is to give a simple statistical mechanism for why this geometry approximately appears. Rather than postulating hierarchical orthogonality from functional desiderata, we show that hierarchical geometry arises naturally from co-occurrence statistics, i.e., how often words appear near one another in text. This account is also more predictive: it implies not only that the same geometry should appear outside LLMs, in simple word embeddings such as word2vec, but also that the geometry should have a specific coarse-to-fine spectral organization.

To see this, we assume and confirm that words closer on the WordNet graph tend to co-occur more often, as illustrated in \cref{fig:figure1}(b): for example, \textit{tree} and \textit{plant} co-occur more often than \textit{tree} and \textit{organism}. Because word2vec embeddings are determined by co-occurrence statistics, this assumption leads to detailed predictions for their geometry. Extracting symmetric subtrees of WordNet, we predict that the corresponding word2vec embeddings exhibit what we call \emph{hierarchical splitting geometry}: successive principal components separate subtrees from coarse to fine levels of the taxonomy. These PCs, which correspond to successive eigendirections of the Gram matrix, first separate broad taxonomic branches and then progressively finer subbranches. In the organism tree of \cref{fig:figure1}, the second PC separates animals from plants by taking opposite signs on the two subtrees; the third and fourth PCs, which are degenerate, split birds from fish and flowers from trees. These hierarchical splits are visualized in \cref{fig:figure1}(c): projecting each node's vector onto the relevant PCs separates nodes according to their hierarchical relationships. The remaining modes encode progressively finer splits, such as separating \textit{daisy} from \textit{poppy}.

We first confirm these predictions in word2vec embeddings. We then show that the same hierarchical splitting geometry extends strikingly well to LLM representations. Overall, our results refine the functional picture of Park et al.: hierarchical geometry is real and robust, but it need not arise from exact orthogonal concept directions. Instead, the approximate orthogonality observed in practice can be understood as a consequence of the spectral structure induced by pairwise word statistics. This yields a more mechanistic and quantitatively predictive account of hierarchical geometry in representation spaces.

\section{Related work}

\textbf{Hierarchical representation.} Hierarchical organization of embeddings was investigated theoretically in a supervised setting \cite{saxe2019_semantic,saxe2013_hierarchical}. In this view, branching points in the taxonomy can lead to different features; interesting representations can emerge in machines seeking to predict those. These works do not consider the structure of the data distribution, however, and as such do not apply to self-supervised methods central to word embeddings.

\textbf{Co-occurrence and embedding geometry.} It is well known that word2vec approximately performs a spectral decomposition of the Pointwise Mutual Information (PMI) matrix characterizing word co-occurrence as in \cref{eq:mstar} \cite{levy2014neural}. Two recent works make assumptions on the structure of co-occurrence to predict the embedding geometry of certain relations. Korchinski et al.~\cite{korchinski2025emergence} argue that if words are characterized by discrete attributes (such as gender) that control co-occurrence, a parallelogram structure of the kind \(king-man+woman=queen\) must occur. Instead, if words are characterized by continuous attributes (such as the seasonality of months of the year or the latitude and longitude of cities), embedding geometry is characterized by smooth manifolds such as circles or maps \cite{karkada2026symmetry}. This work extends these approaches to hierarchical semantic notions, in particular hypernymy.

\section{Theory}

\subsection{From co-occurrence statistics to hierarchical embedding geometry}
\label{sec:theory_assumptions}

Word embedding models such as word2vec and GloVe learn from co-occurrence statistics \cite{mikolov2013distributed,pennington2014glove}. Recent work has shown that their learned representations can be expressed in terms of the top eigenvectors of a normalized co-occurrence matrix \(M^\star\) \cite{karkada2025closedform}. This matrix is defined by
\begin{equation}
M^\star_{ij}
:=
\frac{P_{ij}-P_iP_j}{\frac12(P_{ij}+P_iP_j)}
\approx
\log\!\left(\frac{P_{ij}}{P_iP_j}\right),
\label{eq:mstar}
\end{equation}
where the approximation connects \(M^\star\) to PMI when \(P_{ij}/(P_iP_j)\) is close to one.

Writing the eigendecomposition \(M^\star = \Phi \Lambda^\star \Phi^\top\),
let \(\lambda^\star_1,\dots,\lambda^\star_n\) denote the eigenvalues in decreasing order, and let \(\phi_1,\dots,\phi_n\) be the corresponding eigenvectors. For an embedding of dimension \(d\), the word2vec representation is constructed from the top \(d\) eigenvectors of \(M^\star\). Concretely, if \(\Phi^{(d)}\in\mathbb R^{|\mathcal S|\times d}\) contains these eigenvectors and \(\Lambda^{\star(d)}\) is the corresponding diagonal matrix of eigenvalues, then the embedding coordinates are
\begin{equation}
\label{eq:w2v_form}
W^{(d)}_{i\mu}
=
\Phi^{(d)}_{i\mu}
\sqrt{\Lambda^{\star(d)}_{\mu\mu}},
\qquad
\mu=1,\dots,d.
\end{equation}
Thus, the embedding dimension \(d\) determines how many positive spectral modes of \(M^\star\) are retained: increasing \(d\) incorporates progressively finer eigenmodes of the co-occurrence structure. The retained eigenvectors define the principal directions of the embedding point cloud, while the corresponding eigenvalues determine the variance along those directions.

Our starting assumption is that words that are closer in the WordNet hypernym graph tend to co-occur in text more often. We take the WordNet hypernym graph as our operational notion of semantic hierarchy and ask what geometric structure this induces in word2vec embeddings.

\begin{assumption}[Hierarchical co-occurrence model]
\label{ass:coocc}
Let \(\mathcal S\) denote the selected vocabulary. Each word is identified with a node in the WordNet hypernym graph, and for \(i,j \in \mathcal S\), let \(\mathrm{dist}(i,j)\) denote their distance. We assume that
\begin{equation}
P_{ij} = P_i P_j \, \left[ \widetilde C(\mathrm{dist}(i,j)) + \varepsilon_{ij} \right],
\end{equation}
where \(P_{ij}\) is the co-occurrence probability within a fixed context window, \(P_i\) is the unigram probability, \(\widetilde C\) is a function of distance, and \(\varepsilon_{ij}\) is a fluctuation term.
\end{assumption}

This assumption formalizes the idea that, after normalizing for unigram frequencies, the mean co-occurrence signal is controlled primarily by distance in the semantic hierarchy. In the analysis below, we neglect the fluctuation term and work with the mean co-occurrence structure, which isolates the effect of distance and yields a tractable model. The effect of noise is discussed later.

Under Assumption~\ref{ass:coocc}, there exists a function \(f\) such that
\begin{equation}
M^\star_{ij} = f(\mathrm{dist}(i,j)).
\label{eq:mstar_tree}
\end{equation}
After neglecting \(\varepsilon_{ij}\), the ratio \(P_{ij}/P_iP_j\) depends only on distance, and therefore \(M^\star_{ij}\), being a deterministic transform of that ratio, also depends only on distance. The kernel \(f\) thus describes how distance on the tree affects co-occurrence.

Empirically, we observe that the mean co-occurrence value is positive at small distances and decays rapidly towards zero within a standard error in \cref{fig:global_decay}. This motivates the following assumption on the function \(f\).

\begin{assumption}[Positive decaying kernel]
\label{ass:decay}
The function \(f\) is positive and strictly decreasing as a function of distance.
\end{assumption}

\begin{figure}[!htbp]
    \centering
    \includegraphics[width=0.45\textwidth]{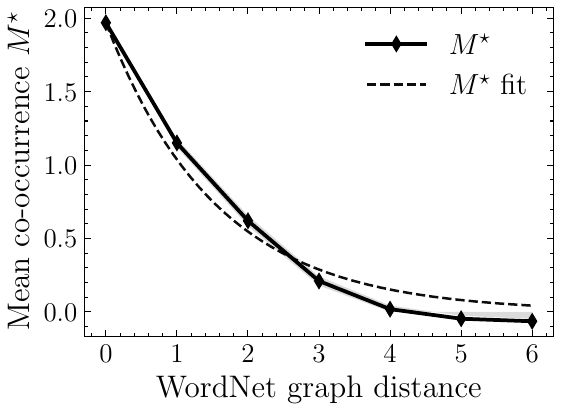}
    \hspace{0.02\textwidth}
    \includegraphics[width=0.45\textwidth]{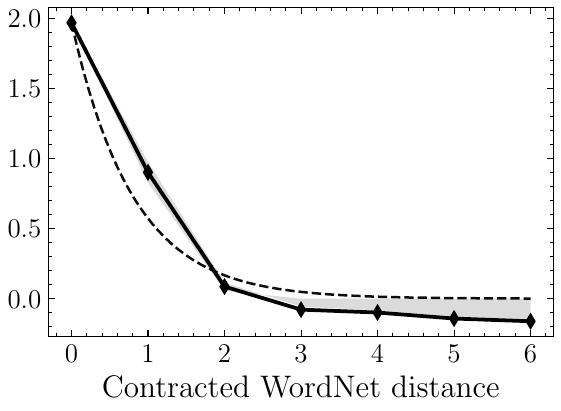}
    \caption{
    \textbf{Mean co-occurrence statistic \(M^\star_{ij}\) decays with semantic distance under two taxonomy constructions.}
    Both the original WordNet-distance measure (left) and the contracted arborescence used in our experiments (right) show monotone, approximately exponential decay. Dashed curves show fitted exponential kernels \(f(d)=\alpha e^{-\beta d}\); shaded bands denote one standard error.
    }
    \label{fig:global_decay}
\end{figure}

We first consider the idealized setting in which \(d \ge \rank(M^\star)\) and \(M^\star\) is positive semidefinite (PSD), so that all nonzero eigenmodes of \(M^\star\) are retained, and the Gram matrix of the embeddings equals \(M^\star\). In practice, however, one typically works with \(d \le \rank(M^\star)\), so the embedding is a rank-\(d\) truncation built from the top \(d\) eigenvectors. \cref{app:m_star_generalize} shows that the same qualitative predictions extend to non-idealized settings where \(M^\star\) is indefinite; empirically, truncating the negative eigenvalues to zero preserves the decay law across the remaining positive eigenspectrum.

Our theory treats the word embeddings associated with any regular subtree of WordNet. We focus on perfect binary subtrees, as many of them can be constructed, allowing us to quantitatively test our theory. Theoretical extensions to more general regular trees are deferred to \cref{app:generalize_trees}.

\subsection{Hierarchy-adapted spectral decomposition}
\label{sec:spectral}

\textbf{Observation:} In the idealized binary-tree setting under the assumptions of \cref{sec:theory_assumptions}, the leading eigenvectors of \(M^\star\) organize hierarchically: the first mode varies only with depth, while subsequent modes split subtrees from coarse to fine levels, as illustrated in the top row of \cref{fig:organism_eigvec}. We now explain why this structure arises.

\begin{figure}[!htbp]
    \centering
    \includegraphics[width=1.\linewidth]{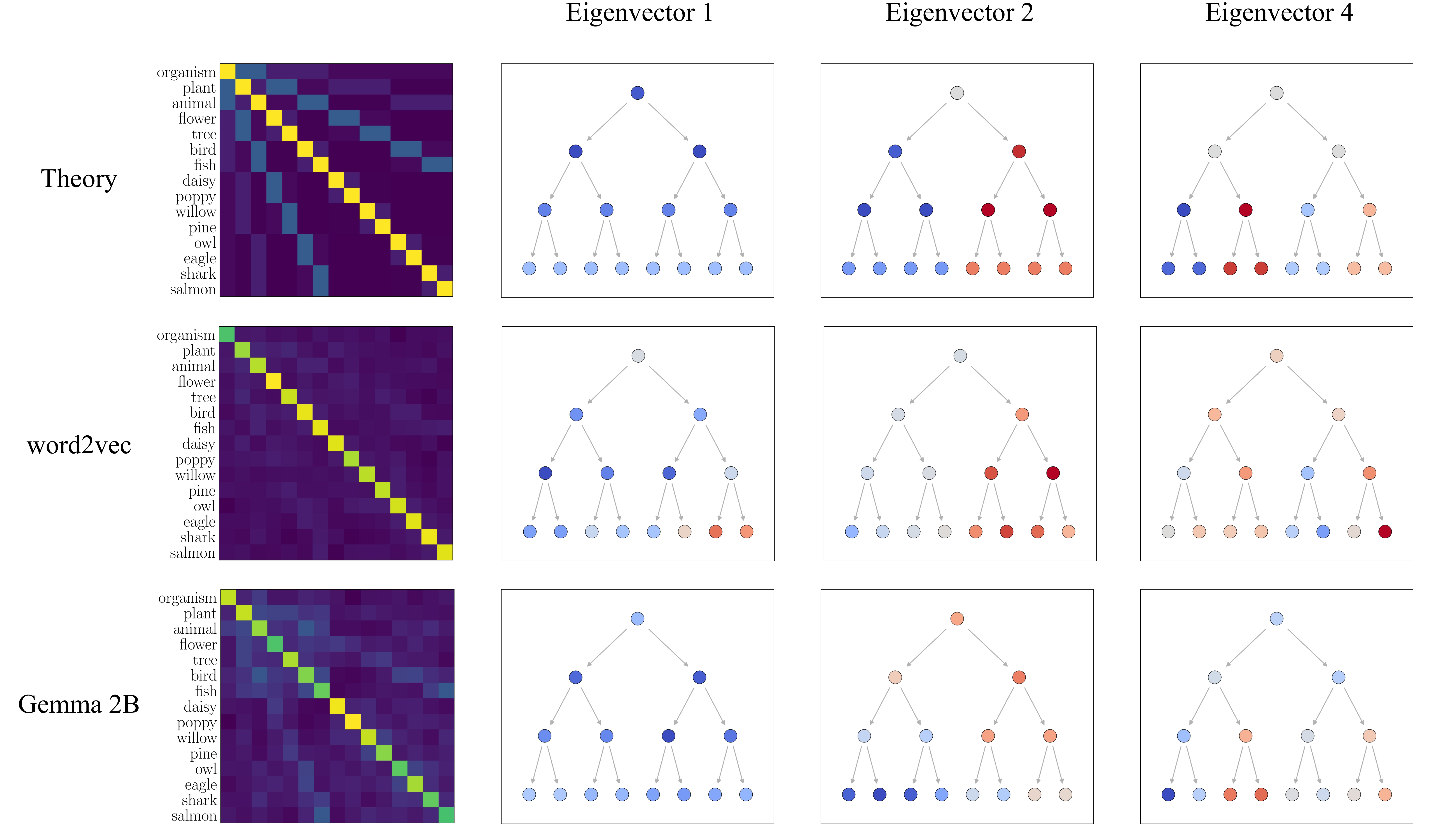}
    \caption{\textbf{Hierarchical splitting geometry in word2vec and LLM representation vectors; example: taxonomy of organisms} (Left) Gram matrices from theory: a fitted exponential kernel \(f(d)=1.967 \cdot e^{-1.235 \cdot  d}\)  (top), word2vec (middle), and Gemma (bottom). Inspection of the eigenstructure reveals qualitative agreement with theory. (Right) Top eigenvectors of each Gram matrix, visualized by projecting each node's representation and coloring by sign (blue: positive; red: negative; white: neutral). The eigenvectors show hierarchical splitting geometry: the first mode is approximately constant; the second splits animals from plants; the third and fourth, which are degenerate, split finer branches. One direction within each degenerate eigenspace is selected for visualization.}
    \label{fig:organism_eigvec}
\end{figure}

Consider a binary tree of depth \(L\), and identify the vocabulary \(\mathcal S\) with its nodes.

\paragraph{Multi-resolution structure on the tree.}

Functions on the nodes of a tree admit a natural multi-resolution decomposition, analogous to Haar wavelets on the line. This basis consists of:

\textit{Scaling basis modes.}
For each level \(\ell=0,1,\dots, L\), let \(\phi_\ell\) denote the normalized indicator of the set of nodes at depth \(\ell\):  \(\phi_\ell\) is constant on level \(\ell\), zero elsewhere, and has unit norm. We refer to \(\phi_\ell\) as the \emph{scaling basis mode at level \(\ell\)}. 

\textit{Wavelet basis modes.}
Fix an internal node \(u\), and let \(T(u)\) denote the subtree rooted at \(u\). For each relative depth \(r=1,\dots,h(u)\), where \(h(u)\) is the height of \(T(u)\), let \(\psi_{u,r}\) denote the normalized Haar wavelet on \(T(u)\) at relative depth \(r\): it is supported on the descendants of \(u\) at that level, has constant magnitude on that support, takes opposite signs on the two sides of the split at \(u\), and has unit norm. We refer to \(\psi_{u,r}\) as the \emph{wavelet basis mode at node \(u\) and relative level \(r\)}. These modes capture hierarchical contrasts inside the subtree rooted at \(u\).

\paragraph{Examples.}

On the organism tree in \cref{fig:figure1}(a), a scaling basis mode assigns the same value to all nodes on a fixed level; for example, the leaf-level scaling basis mode is
\[
\phi_{\ell=2} = \left(e_{\text{flower}}+e_{\text{tree}}+e_{\text{bird}}+e_{\text{fish}}\right)\ /\ 2.
\]
A wavelet basis mode instead contrasts two sibling subtrees; for example,
\[
\psi_{u=\text{organism},r=2}= \left(e_{\text{flower}}+e_{\text{tree}}-e_{\text{bird}}-e_{\text{fish}}\right)\ /\ 2.
\]

\paragraph{Block structure of \(M^\star\).}

Under Assumption~\ref{ass:coocc}, \(M^\star\) is invariant under swapping the two child subtrees of any internal node. To state the resulting invariant subspaces explicitly, define the scaling space
\begin{equation}
    \label{eq:scaling_space}
S_{\mathrm{sc}}:=\operatorname{span}\{\phi_{\ell=0},\dots,\phi_{\ell=L}\},
\end{equation}
and, for each internal node \(u\), define the split space
\begin{equation}
    \label{eq:split_space}
    S^u:=\operatorname{span}\{\psi_{u,r=1},\dots,\psi_{u,r=h(u)}\}.
\end{equation}

\(M^\star\) preserves these subspaces. Multiplication by \(M^\star\) sends a vector \(x\) to \((M^\star x)_i=\sum_j f(\mathrm{dist}(i,j))x_j\). Two nodes at the same depth have equal distance profiles with respect to the other nodes and thus transform equally. If \(x=\psi_{u,r}\), then for any node outside \(T(u)\), the two child subtrees of \(u\) contribute equal weights with opposite signs and cancel; inside \(T(u)\), the result remains antisymmetric across the same split and depends only on relative depth. Hence
\[
M^\star S_{\mathrm{sc}}\subseteq S_{\mathrm{sc}},
\qquad
M^\star S^u\subseteq S^u.
\]

Thus, in the hierarchy-adapted Haar basis, \(M^\star\) decomposes into a scaling block \(A^{\mathrm{sc}}\in\mathbb R^{(L+1)\times(L+1)}\) on \(S_{\mathrm{sc}}\) and wavelet blocks \(A^u\in\mathbb R^{h(u)\times h(u)}\) on each \(S^u\). By symmetry, subtrees of the same height \(h\) have identical wavelet blocks, denoted \(A^{(h)}\). This gives the following structural result.

\begin{theorem}[Hierarchy-aligned eigenvectors]
\label{thm:eigenstructure}
The eigenvectors of \(M^\star\) can be chosen to lie entirely within either the scaling subspace or a wavelet subspace associated with the subtree of height \(h\). Specifically, the \(k\)-th scaling eigenvector has the form

\begin{equation}
\label{eq:scaling_eig}
v^{\mathrm{sc}}_k=\sum_{\ell=0}^L c_{\ell,k}\phi_\ell,
\end{equation}
while the \(k\)-th eigenvector in the wavelet subspace associated with an internal node \(u\) of height \(h\) has the form
\begin{equation}
\label{eq:wavelet_eig}
v^{(u)}_k=\sum_{r=1}^{h} a^{(h)}_{r,k}\psi_{u,r}.
\end{equation}
\end{theorem}

\paragraph{Ordering of eigenvalues.}

We now use Assumption~\ref{ass:decay} to constrain the ordering of these modes and further characterize the dominant ones. We write \(\lambda^{\mathrm{sc}}_i\) for the \(i\)-th largest eigenvalue of \(A^{\mathrm{sc}}\), and \(\lambda^{(h)}_i \) for the \(i\)-th largest eigenvalue of \(A^{(h)}\).

\begin{theorem}[Coarse-to-fine spectral ordering]
\label{thm:ordering}
Under Assumption~\ref{ass:decay}:

\begin{enumerate}
    \item The largest eigenvalue of \(M^\star\) is \(\lambda^{\mathrm{sc}}_1\). The corresponding eigenvector can be written as in \cref{eq:scaling_eig} with \(c_{\ell,1}>0\) for all \(\ell\). 

    \item The leading eigenvector of each wavelet block \(A^{(h)}\) can be written as in \cref{eq:wavelet_eig} 
    with \(a^{(h)}_{r,1}>0\) for all \(r\).

    \item The wavelet blocks are nested across subtree height: for each \(h\geq 1\), \(A^{(h)}\) is exactly the upper-left \(h\times h\) corner of \(A^{(h+1)}\). Therefore, \(A^{(h)}\) is a principal submatrix of \(A^{(h+1)}\). Consequently, their eigenvalues satisfy Cauchy interlacing:
\[
\lambda^{(h+1)}_1 \ge \lambda^{(h)}_1 \ge \cdots \ge \lambda^{(h+1)}_h \ge \lambda^{(h)}_h \ge \lambda^{(h+1)}_{h+1}.
\]

    \item In particular, the leading wavelet eigenvalues are monotone in subtree height:
    \[
    \lambda^{(L)}_1 \ge \lambda^{(L-1)}_1 \ge \cdots \ge \lambda^{(1)}_1.
    \]
    Thus, dominant non-scaling modes appear coarse-to-fine in the spectrum.
\end{enumerate}
\end{theorem}

\paragraph{Interpretation.}

The top row of \cref{fig:organism_eigvec} shows the hierarchical splitting geometry defined by \cref{thm:eigenstructure,thm:ordering}. The first mode is a scaling eigenvector and only varies along depth, so it is less informative about the branching structure. The next modes are the wavelet eigenvectors that resolve the hierarchy from coarse to fine. First comes the root split, then the splits associated with its children, and so on down the tree. In this sense, the top of the spectrum reveals the major taxonomic contrasts before finer within-subtree distinctions.

\paragraph{Degenerate eigenspaces.}
For subtrees of the same height, dominant wavelet modes are degenerate by symmetry, so individual eigenvectors are identifiable only up to rotation within the degenerate eigenspace. This explains why empirical eigenvectors may appear as mixtures of same-level splits, such as the plant- and animal-rooted wavelets in \cref{fig:organism_eigvec}. Proofs of each statement are given in \cref{app:proofs}.

\subsection{Parametric decay model}
\label{sec:parametric_decay}

The results above characterize the qualitative spectral structure under general decay assumptions. To obtain quantitative predictions, we instantiate the kernel \(f\) with a simple parametric form. The parametric model resolves the relative placement of the non-leading scaling modes, which is not fixed by the qualitative theory alone.

 As shown in \cref{fig:global_decay}, mean normalized co-occurrence decays approximately exponentially in distance over the range relevant to our experiments. We therefore model the decay with the following parametric form:
\begin{equation}
    \label{eq:exp_decay}
    f(\mathrm{dist}(i,j)) = \alpha \cdot e^{ -\beta\cdot \mathrm{dist}(i,j)}
\end{equation}
with fitted parameters \(\alpha,\ \beta > 0\).
This form satisfies the positivity and monotonic-decay assumptions of the preceding analysis while providing a flexible two-parameter family that captures the observed behavior.
\subsection{Effect of noise}

In the idealized setting, the hierarchical splitting geometry is exact, but empirical co-occurrence statistics contain fluctuations. Such noise can be treated by regular perturbation theory: it lifts exact degeneracies and mixes directions within nearly degenerate eigenspaces. If small enough, it does not destroy the underlying coarse-to-fine structure. In our empirical tests, we thus use observables designed to be resilient to these effects; these observables are defined in \cref{sec:alignment_metric}.

\section{Empirical testing in word2vec embeddings}

\subsection{Empirical setup}

\subsubsection{Vocabulary and hierarchy construction}
\label{sec:4.1}

We restrict our consideration to words satisfying three conditions: (i) a matching noun word/lemma in WordNet, (ii) a matching string in Gemma's token vocabulary so that our predictions are testable in this LLM, and (iii) no polysemy within the candidate set. The last condition ensures that each retained word has an unambiguous word--sense assignment, so that its hypernym relations are well defined.

To obtain a tree-compatible hierarchy, we convert the WordNet noun hypernym DAG into a rooted arborescence by selecting parent assignments that maximize depth from the global root. We then restrict to eligible words and contract paths through ineligible intermediate nodes. The resulting contracted arborescence defines the distances used in all empirical tests; construction details are given in \cref{app:w2v_construction}. Under this distance, mean normalized co-occurrence decays approximately exponentially as shown in the right panel of \cref{fig:global_decay}. We fit the exponential kernel in \cref{eq:exp_decay} to this empirical decay curve to use as the theoretical object compared to the empirical embedding Gram matrices below.

\subsubsection{word2vec embeddings}
We diagonalize the restricted matrix \(M^\star\) from \cref{eq:mstar} and define word2vec embeddings from its top \(d=2048\) eigenvectors and eigenvalues as in \cref{eq:w2v_form}. We choose \(d=2048\) to match the Gemma representation dimension, ensuring dimensional consistency across the two settings. Additional construction details are given in \cref{app:w2v_construction}.

\subsubsection{Binary subtree sampling}
\label{sec:binary_tree_construction}

For quantitative tests, we sample perfect binary subtrees of depth \(L=3\). We focus on \(L=3\), which is the largest depth that provides reliable sample sizes under our corpus and vocabulary constraints; \cref{app:l2_repeat} reports analogous results for \(L=2\). A valid tree is defined recursively by choosing two distinct valid child subtrees at each internal node. We sample uniformly over valid tree structures using dynamic programming counts, avoiding bias toward individual word sets.

\subsection{Example: the taxonomy of organisms}

We first compare the word2vec Gram matrix for the organism taxonomy of \cref{fig:figure1} with our theory. The top two rows of \cref{fig:organism_eigvec} compare the Gram matrices and leading eigenvectors from each. 
Our central prediction that the leading eigenvectors organize into scaling and split modes is confirmed in the organism example. The first mode is approximately constant at each level and has a consistent sign, while subsequent modes separate subtrees through sign at progressively finer resolutions---one separating plants from animals, others resolving finer contrasts such as bird versus fish or flower versus tree.

\subsection{Quantitative eigenspace alignment across word2vec}
\label{sec:alignment_metric}

 We now test whether our predictions hold systematically across many binary subtrees of depth \(L=3\). We sample eligible subtrees and compare the empirical Gram matrix from the word2vec embeddings with the fitted exponential Gram matrix.

To quantify alignment, we compare subspaces rather than individual eigenvectors. The theoretical spectrum contains exponentially many degeneracies as tree depth increases, so individual same-level split directions are only defined up to rotation within their eigenspace. Empirical fluctuations further perturb these directions, making eigenvector-level comparisons unstable. We therefore use top-\(k\) eigenspace alignment:
\begin{equation}
    \label{eq:top_k_alignment}
    g(k):=
    \frac{1}{k}\|U_k^\top V_k\|_F^2,
\end{equation}
where \(U_k\) and \(V_k\) contain the top \(k\) eigenvectors of the empirical and theoretical Gram matrices, respectively. This quantity approaches unity when the leading eigenspaces of the two matrices are perfectly aligned.

For organism- and cognition-rooted trees, empirical top-\(k\) eigenspaces align substantially above shuffled-label baselines, obtained by globally permuting the word and vector pairs, as shown in the left and middle panels of \cref{fig:micro_evidence}.

\begin{figure}[!htbp]
    \centering
    \includegraphics[width=1.\linewidth]{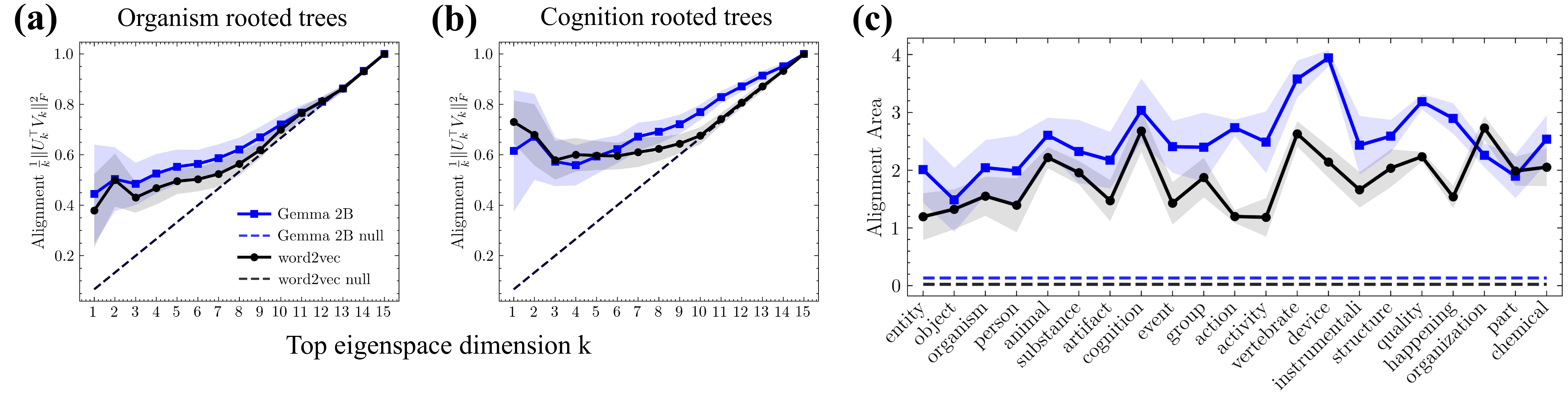}
    \caption{
    \textbf{Hierarchical splitting geometry in binary trees sampled from WordNet.}
    Top-\(k\) eigenspace alignment between theoretical and empirical Gram matrices for (a) organism-rooted trees, and (b) cognition-rooted trees. Panel (c) summarizes all eligible roots using the alignment area from \cref{eq:top_k_alignment_area}. Both word2vec and Gemma unembeddings align substantially above shuffled-label baselines. Shaded bands indicate one standard deviation; each root uses at most 5,000 sampled trees.
    }
    \label{fig:micro_evidence}
\end{figure}

To assess the generality of this finding, we repeat this procedure over all eligible root words with constructible binary trees. For each root, we summarize alignment by the area between the top-\(k\) alignment curve and a null diagonal baseline:
\begin{equation}
    \label{eq:top_k_alignment_area}
    \text{Alignment Area} =\sum^{15}_{k=1}\left (\frac{1}{k}\|U_k^\top V_k\|_F^2 - \frac{k}{15}\right).
\end{equation}
We average this statistic over sampled trees from each root. The right panel of \cref{fig:micro_evidence} shows the distribution across eligible roots together with the shuffled baseline; the empirical alignment area (solid black) is consistently above the null (dashed).

Substantial spectral alignment appears across many semantic hierarchies, not only organism-rooted trees, confirming that a hierarchical splitting geometry is a recurring feature of co-occurrence-based embeddings.

\section{Empirical testing for LLMs}

We now ask whether the same hierarchical splitting geometry appears in LLM representations. We address this in the Gemma 2B model \cite{gemmateam2024gemmaopenmodelsbased}, analyzing its unembedding vectors after centering and whitening to reduce global anisotropy in the representation space.

\subsection{LLM representation processing}

Following Park et al.~\cite{park2024linearrepresentationhypothesisgeometry}, we analyze whitened unembedding vectors \(v_i\):
\(\tilde v_i = \Sigma^{-1/2} (v_i-\mu)\),
where \(\mu\) and \(\Sigma\) are mean and covariance of the full vocabulary unembeddings.  Our conclusions hold similarly with and without whitening. In the main text, we study the former, and in \cref{app:robustness_alternative_embeddings} the latter.

\subsection{Testing our theory across LLM unembeddings}
\label{sec:5.2}

We compare Gram matrices from the processed Gemma unembeddings to those induced by the fitted mean normalized co-occurrence kernel. Remarkably, the organism example in the bottom panel of \cref{fig:organism_eigvec} exhibits a leading scaling mode followed by split modes, as predicted.

The blue curves in all panels of \cref{fig:micro_evidence} quantify this agreement using the top-\(k\) eigenspace alignment (\cref{eq:top_k_alignment,eq:top_k_alignment_area}): across sampled binary trees, the leading eigenspaces of the LLM Gram matrices align substantially better with the theoretical predictions than with the shuffled baseline. Overall, the hierarchical splitting geometry signatures extend remarkably well to the LLM setting.

\paragraph{Connection to previous work:} Beyond eigenspace alignment, we ask whether the same co-occurrence model also explains a concept-level signature previously reported by Park et al.~\cite{park2025geometrycategoricalhierarchicalconcepts}. Park et al. associate each WordNet concept \(w\) with a concept vector \(\bar\ell_w\), intended to represent membership in \(w\) or one of its descendants. Their hierarchical-geometry prediction is that a child concept vector decomposes into a parent component plus a child-specific innovation: for an immediate parent--child pair \((p,w)\), the difference \(\bar\ell_w-\bar\ell_p\) is postulated, for functional reasons, to be orthogonal to the parent vector \(\bar\ell_p\). Thus their diagnostic computes
\[
\cos(\bar\ell_w-\bar\ell_p,\bar\ell_p),
\]
which should concentrate near zero for true parent--child pairs.

In \cref{fig:Parkconsistency}, we apply this same diagnostic both to Gemma unembeddings and to theoretical co-occurrence embeddings constructed from our fitted exponential kernel. In both cases, parent--child innovations are close to (but not exactly, even in our noiseless theoretical setting) orthogonal to the parent vector, while a shuffled-parent baseline is substantially displaced. Thus, the concept-vector orthogonality pattern approximately observed by Park et al. is also reproduced by a purely co-occurrence-driven model. This suggests that the pattern need not be interpreted as evidence for a hierarchy-specific functional mechanism; it can arise from the same spectral structure that produces hierarchical splitting geometry. Details of the concept-vector estimator and parent selection are given in \cref{app:park_consistency}.

\begin{figure}[!htbp]
    \centering
    \includegraphics[width=0.95\linewidth]{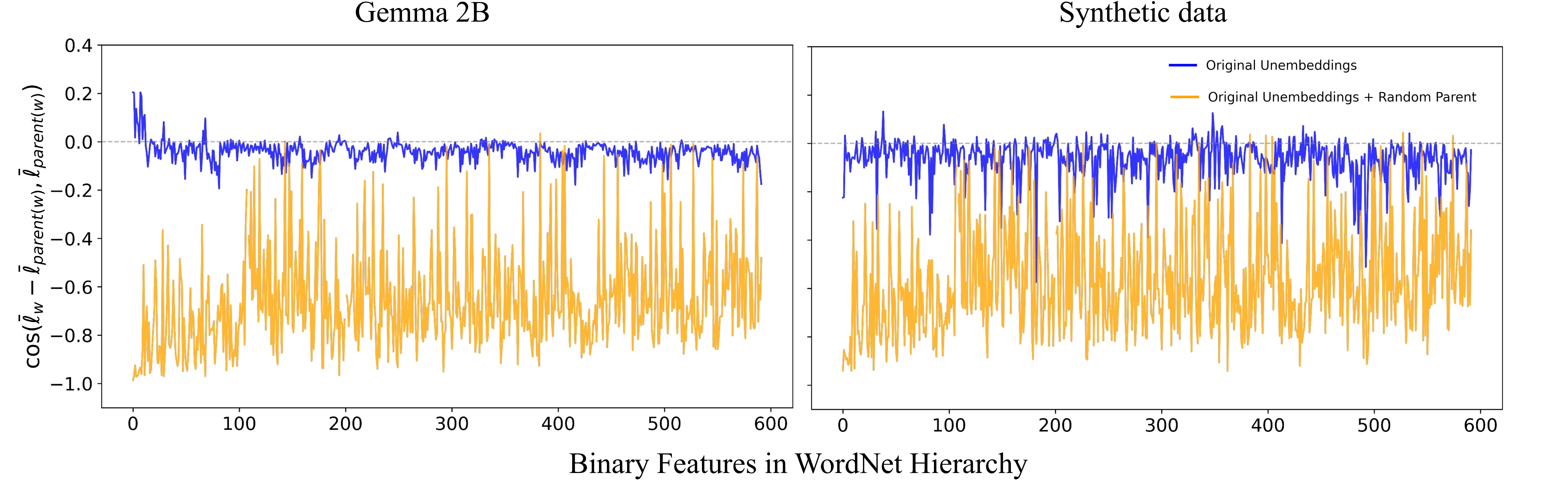}
    
    \caption{
    \textbf{Concept-level orthogonal innovations from co-occurrence statistics.}
    Concept vectors are estimated using 70\% of descendant tokens independently selected for each synset, following Park et al.~\cite{park2025geometrycategoricalhierarchicalconcepts}. Both Gemma unembeddings and theoretical co-occurrence embeddings yield parent--child innovations concentrated near zero, while a shuffled-parent baseline is substantially displaced.
    }
    \label{fig:Parkconsistency}
\end{figure}

\section{Limitations}
\label{sec:limitations}
Our theory applies directly to co-occurrence-driven word2vec-style embeddings, where explicit spectral predictions are available. Although we find the same hierarchical splitting geometry in LLM unembeddings, extending the theory to transformer training dynamics remains a challenge for the field.
It would be particularly interesting to study how LLMs can use context to represent ambiguous words, currently left out of our analysis. 

\section{Conclusion}
We presented a distributional account of hierarchical semantic geometry. Starting from the assumption that word co-occurrence decays with distance on the WordNet hypernym graph, we derived predictions of hierarchical splitting geometry for the word2vec embeddings of any WordNet-induced subtrees. The resulting PCs encode the taxonomy from coarse to fine levels, as confirmed not only in word2vec but also in Gemma unembeddings.
Together with~\cite{korchinski2025emergence,karkada2026symmetry}, this work supports the following broader view: each word is characterized by discrete attributes (such as gender or singular vs.\ plural), continuous attributes (such as season of the year or geographical location), and hierarchical attributes (position in WordNet). Words with similar attributes co-occur more often, and this alone gives rise to the elegant geometrical organization of word embeddings. Such organization may be useful for function, but is not driven by it.

\bibliographystyle{unsrt}

\section*{Acknowledgments}
This work was supported by the Simons Foundation through the Simons Collaboration on the Physics of Learning and Neural Computation (Award ID: SFI-MPS-POL00012574-05), PI Wyart.

\section*{Code availability}
Code for reproducing the experiments and figures will be released in a public repository.

\bibliography{references}

\newpage
\appendix
\section{Proofs for the hierarchy-aligned spectral theory}
\label[app]{app:proofs}

This appendix gives the formal proofs for the theoretical claims in \cref{sec:theory_assumptions,sec:spectral,sec:parametric_decay}.  The purpose of the organization below is to make explicit which assumptions are used in each main-text result.

\paragraph{Proof map.}
\Cref{app:distance_kernel_reduction} justifies the reduction from the co-occurrence model in \cref{ass:coocc} to a distance kernel.  \Cref{app:binary_tree_setup} fixes the idealized binary-tree setting used in \cref{thm:eigenstructure,thm:ordering}.  \Cref{app:proof_thm_eigenstructure} proves \cref{thm:eigenstructure}.  \Cref{app:proof_thm_ordering} proves each item of \cref{thm:ordering}.  \Cref{app:exp_kernel_proof} records the additional algebra used by the exponential model in \cref{sec:parametric_decay}.  \Cref{app:perturbation_comment} explains the perturbative statement made in the noise discussion.

\paragraph{Assumption map.}
The following table records the assumptions used by each result.

\begin{center}
\small
\setlength{\tabcolsep}{4pt}
\renewcommand{\arraystretch}{1.15}
\begin{tabularx}{\linewidth}{@{}p{0.24\linewidth}X X@{}}
\toprule
Result & Main-text claim & Assumptions used \\
\midrule
\Cref{prop:distance_kernel_reduction}
& \cref{eq:mstar_tree}
& \cref{ass:coocc} with $\varepsilon_{ij}=0$ \\

\Cref{thm:eigenstructure}
& hierarchy-aligned eigenvectors
& full binary tree; $M_{ij}=f(\mathrm{dist}(i,j))$ \\

\Cref{thm:ordering}, item 1
& top eigenvector is scaling and positive
& full binary tree; $f(d)>0$ \\

\Cref{thm:ordering}, item 2
& leading wavelet eigenvectors are sign-coherent
& full binary tree; $f$ strictly decreasing \\

\Cref{thm:ordering}, items 3--4
& nested split blocks and coarse-to-fine ordering
& full binary tree; $M_{ij}=f(\mathrm{dist}(i,j))$ \\

\Cref{prop:exp_rank_one}
& exponential-kernel interlacing
& $f(d)=\alpha e^{-\beta d}$ with $\alpha,\beta>0$ \\
\bottomrule
\end{tabularx}
\end{center}

\subsection{From the co-occurrence model to a distance kernel}
\label[app]{app:distance_kernel_reduction}

\begin{proposition}[Distance-kernel reduction]
\label{prop:distance_kernel_reduction}
Under \cref{ass:coocc}, if the fluctuation term is neglected, i.e. $\varepsilon_{ij}=0$, then there exists a scalar function $f$ such that
\begin{equation}
    M^\star_{ij}=f(\mathrm{dist}(i,j)).
\end{equation}
Moreover, if the mean normalized co-occurrence is positive and decreasing as a function of distance, then $f$ satisfies \cref{ass:decay}.
\end{proposition}

\begin{proof}
With $\varepsilon_{ij}=0$, \cref{ass:coocc} gives
\begin{equation}
    \frac{P_{ij}}{P_iP_j}=\widetilde C(\mathrm{dist}(i,j)).
\end{equation}
The statistic in \cref{eq:mstar} is a deterministic scalar transform of this ratio:
\begin{equation}
    M^\star_{ij}
    =
    \frac{P_{ij}-P_iP_j}{\frac12(P_{ij}+P_iP_j)}
    =
    \frac{\widetilde C(\mathrm{dist}(i,j))-1}{\frac12(\widetilde C(\mathrm{dist}(i,j))+1)}.
\end{equation}
Thus $M^\star_{ij}=f(\mathrm{dist}(i,j))$ with
\begin{equation}
    f(d):=\frac{\widetilde C(d)-1}{\frac12(\widetilde C(d)+1)}.
\end{equation}
The final statement is exactly the empirical positivity and monotonicity condition stated as \cref{ass:decay}.
\end{proof}

\subsection{Binary-tree notation and Haar basis}
\label[app]{app:binary_tree_setup}

Let $T_L$ be the full binary tree of depth $L$, containing all nodes at depths $0,1,\ldots,L$.  Let $V$ be the real vector space of functions on the nodes of $T_L$, with standard basis vectors $e_i$.  For a node $u$, let $|u|$ denote its depth, let $h(u):=L-|u|$ denote the height of the subtree rooted at $u$, and let $T_r(u)$ denote the descendants of $u$ at relative depth $r$.  Thus $|T_r(u)|=2^r$.

The scaling modes are
\begin{equation}
\label{eq:app_scaling_mode}
    \phi_\ell
    :=
    2^{-\ell/2}\sum_{i\in T_\ell(0)} e_i,
    \qquad \ell=0,1,\ldots,L.
\end{equation}
If $u$ is internal, write its two children as $u_+$ and $u_-$.  The wavelet modes rooted at $u$ are
\begin{equation}
\label{eq:app_wavelet_mode}
    \psi_{u,r}
    :=
    2^{-r/2}
    \left(
    \sum_{i\in T_{r-1}(u_+)} e_i
    -
    \sum_{i\in T_{r-1}(u_-)} e_i
    \right),
    \qquad r=1,\ldots,h(u).
\end{equation}
Equivalently, $\psi_{u,r}$ is supported on $T_r(u)$, is constant on each side of the split at $u$, takes opposite signs on the two sides, and has unit norm.

Define
\begin{equation}
    S_{\mathrm{sc}}:=\operatorname{span}\{\phi_0,\ldots,\phi_L\},
    \qquad
    S^u:=\operatorname{span}\{\psi_{u,1},\ldots,\psi_{u,h(u)}\}.
\end{equation}
These are the spaces defined in \cref{eq:scaling_space,eq:split_space}.

\begin{lemma}[Orthogonal Haar decomposition]
\label{lem:haar_basis}
The collection
\begin{equation}
    \mathcal H
    :=
    \{\phi_0,\ldots,\phi_L\}
    \cup
    \{\psi_{u,r}: u \text{ internal},\ r=1,\ldots,h(u)\}
\end{equation}
is an orthonormal basis of $V$.  Equivalently,
\begin{equation}
\label{eq:app_orthogonal_decomp}
    V
    =
    S_{\mathrm{sc}}
    \oplus
    \bigoplus_{u\ \mathrm{internal}} S^u .
\end{equation}
\end{lemma}

\begin{proof}
The normalization in \cref{eq:app_scaling_mode,eq:app_wavelet_mode} gives unit norm.  Scaling modes at distinct depths have disjoint support, so they are orthogonal.  A scaling mode and a wavelet mode either have disjoint support or lie on the same depth layer; in the latter case the inner product is proportional to the sum of the $+$ and $-$ coefficients of the wavelet, which cancel.

Now consider two wavelet modes.  If their supports are on different depth layers, their inner product is zero.  If their rooted subtrees are disjoint, their supports are disjoint.  If one root is an ancestor of the other, then the ancestor wavelet is constant on the descendant subtree at the relevant depth, while the descendant wavelet has zero sum across the two children of its root; hence their inner product is zero.  Finally, if they have the same root, then distinct relative depths have disjoint support and the same relative depth gives unit norm.

It remains only to count dimensions.  The number of modes is
\begin{equation}
    (L+1)+\sum_{d=0}^{L-1}2^d(L-d)
    =
    2^{L+1}-1,
\end{equation}
which equals the number of nodes of $T_L$.  Hence the orthonormal collection is a basis.
\end{proof}

Throughout the rest of this appendix, let
\begin{equation}
\label{eq:app_distance_kernel}
    M_{ij}=f(\mathrm{dist}(i,j))
\end{equation}
be the idealized distance-kernel matrix on $T_L$.  In the main text, this matrix is $M^\star$ after applying \cref{prop:distance_kernel_reduction}.

\subsection{Proof of \cref{thm:eigenstructure}: hierarchy-aligned eigenvectors}
\label[app]{app:proof_thm_eigenstructure}

\begin{lemma}[Invariance of the scaling and split spaces]
\label{lem:block_invariance}
For $M$ defined by \cref{eq:app_distance_kernel},
\begin{equation}
    M S_{\mathrm{sc}}\subseteq S_{\mathrm{sc}},
    \qquad
    M S^u\subseteq S^u
    \quad\text{for every internal node }u.
\end{equation}
Consequently, in the basis $\mathcal H$, $M$ is block diagonal with one block on $S_{\mathrm{sc}}$ and one block on each $S^u$.
\end{lemma}

\begin{proof}
First let $x\in S_{\mathrm{sc}}$.  Then $x$ is constant on every depth layer.  If $i$ and $i'$ have the same depth, there is a tree automorphism sending $i$ to $i'$ and preserving all depth layers.  Since $\mathrm{dist}(i,j)$ is invariant under tree automorphisms, the sums
\begin{equation}
    (Mx)_i=\sum_j f(\mathrm{dist}(i,j))x_j,
    \qquad
    (Mx)_{i'}=\sum_j f(\mathrm{dist}(i',j))x_j
\end{equation}
are equal.  Hence $Mx$ is also constant on depth layers, so $Mx\in S_{\mathrm{sc}}$.

Now fix an internal node $u$ and let $x\in S^u$.  Such an $x$ is supported on $T(u)$, is antisymmetric under the swap of the two children of $u$, and is constant on each relative depth layer within each child subtree.  If $i\notin T(u)$, then every node $j$ on the $+$ side of $u$ can be paired with the corresponding node $j'$ on the $-$ side.  These paired nodes satisfy $\mathrm{dist}(i,j)=\mathrm{dist}(i,j')$ and $x_j=-x_{j'}$, so their contributions cancel.  Thus $(Mx)_i=0$ for all $i\notin T(u)$.

For $i\in T(u)$, the same child-swap symmetry gives antisymmetry of $Mx$ across the two sides of the split at $u$.  Moreover, automorphisms acting inside either child subtree show that $(Mx)_i$ depends only on the side of $u$ containing $i$ and on the relative depth of $i$ below $u$.  Therefore $Mx$ is supported on $T(u)$, is antisymmetric across the split at $u$, and is constant on relative depth layers within each side.  This is exactly the space $S^u$.
\end{proof}

\begin{proof}[Proof of \cref{thm:eigenstructure}]
By \cref{lem:haar_basis}, $V$ decomposes orthogonally as
\begin{equation}
    V=S_{\mathrm{sc}}\oplus\bigoplus_{u\ \mathrm{internal}}S^u.
\end{equation}
By \cref{lem:block_invariance}, each summand is invariant under the symmetric matrix $M$.  Therefore $M$ restricts to a symmetric linear operator on each summand, and each restriction admits an orthonormal eigenbasis.  Taking the union of these eigenbases gives an eigenbasis of $M$ whose vectors lie entirely within the scaling space or within a single split space.

The eigenvectors in the scaling block therefore have the form
\begin{equation}
    v^{\mathrm{sc}}_k=\sum_{\ell=0}^L c_{\ell,k}\phi_\ell,
\end{equation}
and the eigenvectors in the split block rooted at $u$ have the form
\begin{equation}
    v^{(u)}_k=\sum_{r=1}^{h(u)} a^{(h(u))}_{r,k}\psi_{u,r}.
\end{equation}
These are exactly \cref{eq:scaling_eig,eq:wavelet_eig}.
\end{proof}

\subsection{Closed-form block entries}
\label[app]{app:closed_form_blocks}

This subsection records the block formulas used in the proof of \cref{thm:ordering}.  Let $A^{\mathrm{sc}}$ be the matrix of $M|_{S_{\mathrm{sc}}}$ in the basis $(\phi_0,\ldots,\phi_L)$, and let $A^{(h)}$ be the common matrix of $M|_{S^u}$ for any node $u$ of height $h$, in the basis $(\psi_{u,1},\ldots,\psi_{u,h})$.

\begin{lemma}[Scaling block]
\label{lem:scaling_block_formula}
For $\ell,m=0,\ldots,L$,
\begin{equation}
\label{eq:app_closed_scaling_form}
A^{\mathrm{sc}}_{\ell,m}
=
2^{-(\ell+m)/2}
\left[
\sum_{a=0}^{\min(\ell,m)-1}
2^{\ell+m-a-1}f(\ell+m-2a)
+
2^{\max(\ell,m)}f(|\ell-m|)
\right],
\end{equation}
where the empty sum is interpreted as zero.
\end{lemma}

\begin{proof}
By definition,
\begin{equation}
    A^{\mathrm{sc}}_{\ell,m}
    =
    2^{-(\ell+m)/2}
    \sum_{i\in T_\ell(0)}\sum_{j\in T_m(0)}f(\mathrm{dist}(i,j)).
\end{equation}
Group ordered pairs $(i,j)$ by the depth $a$ of their lowest common ancestor.  If the lowest common ancestor is strictly above both nodes, then $a=0,\ldots,\min(\ell,m)-1$, the distance is $\ell+m-2a$, and the number of ordered pairs with that value of $a$ is $2^{\ell+m-a-1}$.  The remaining pairs are those in which one node is ancestral to the other, including $i=j$ when $\ell=m$.  Their distance is $|\ell-m|$, and their number is $2^{\max(\ell,m)}$.  Substitution gives \cref{eq:app_closed_scaling_form}.
\end{proof}

\begin{lemma}[Split block]
\label{lem:split_block_formula}
For $r,s=1,\ldots,h$,
\begin{equation}
\label{eq:app_closed_split_form}
A^{(h)}_{r,s}
=
2^{-(r+s)/2}
\left[
-
2^{r+s-1}f(r+s)
+
\sum_{a=1}^{\min(r,s)-1}
2^{r+s-a-1}f(r+s-2a)
+
2^{\max(r,s)}f(|r-s|)
\right].
\end{equation}
\end{lemma}

\begin{proof}
Fix a node $u$ of height $h$.  In the bilinear form
\begin{equation}
    \psi_{u,r}^\top M\psi_{u,s},
\end{equation}
pairs on opposite sides of the split at $u$ enter with negative sign, while pairs on the same side enter with positive sign.  There are $2^{r+s-1}$ ordered pairs on opposite sides, and every such pair has distance $r+s$, giving the first term.

For pairs on the same side of the split, group by the depth $a$ of the lowest common ancestor relative to $u$.  If the lowest common ancestor lies strictly below $u$ and strictly above both nodes, then $a=1,\ldots,\min(r,s)-1$, the distance is $r+s-2a$, and the number of ordered pairs is $2^{r+s-a-1}$.  The remaining same-side pairs are those in which one node is ancestral to the other, including the diagonal case, and there are $2^{\max(r,s)}$ such ordered pairs at distance $|r-s|$.  Including the normalization of the two wavelets gives \cref{eq:app_closed_split_form}.
\end{proof}

\subsection{Proof of \cref{thm:ordering}: coarse-to-fine spectral ordering}
\label[app]{app:proof_thm_ordering}

We prove the four items of \cref{thm:ordering} in order.

\paragraph{Item 1: the largest eigenvalue lies in the scaling block.}
Assume $f(d)>0$ for all relevant distances.  Then $M$ is a strictly positive symmetric matrix.  By the Perron--Frobenius theorem, its largest eigenvalue is simple and has an eigenvector $v_{\mathrm{PF}}$ with strictly positive entries.

By \cref{thm:eigenstructure}, every eigenvector of $M$ can be chosen to lie in one of the blocks in \cref{eq:app_orthogonal_decomp}.  Since the Perron--Frobenius eigenvalue is simple, $v_{\mathrm{PF}}$ itself must lie entirely in one of these invariant summands.  It cannot lie in any split space $S^u$: if $u$ is not the root, every vector in $S^u$ vanishes outside $T(u)$, and if $u$ is the root, every nonzero vector in $S^u$ has opposite signs on the two sides of the root split.  Therefore $v_{\mathrm{PF}}\in S_{\mathrm{sc}}$.

Writing
\begin{equation}
    v_{\mathrm{PF}}=\sum_{\ell=0}^L c_{\ell,1}\phi_\ell,
\end{equation}
we have $c_{\ell,1}>0$ for all $\ell$, because $v_{\mathrm{PF}}$ is strictly positive and the supports of the $\phi_\ell$ are disjoint depth layers.  Hence the largest eigenvalue of $M$ is the largest eigenvalue $\lambda^{\mathrm{sc}}_1$ of the scaling block $A^{\mathrm{sc}}$, and its eigenvector has the form stated in item 1.

\paragraph{Item 2: the leading eigenvector of each split block is sign-coherent.}
Assume now that $f$ is strictly decreasing.  We show first that every entry of every split block $A^{(h)}$ is strictly positive.  In \cref{eq:app_closed_split_form}, the negative contribution comes from opposite-side pairs, all at distance $r+s$.  The positive contributions come from same-side pairs.  The total coefficient of the same-side terms equals the total coefficient of the opposite-side term:
\begin{equation}
    \sum_{a=1}^{\min(r,s)-1}2^{r+s-a-1}+2^{\max(r,s)}
    =
    2^{r+s-1}.
\end{equation}
Every same-side distance appearing in the positive terms is strictly smaller than $r+s$.  Since $f$ is strictly decreasing, each same-side contribution is larger than the corresponding contribution at distance $r+s$.  Therefore
\begin{equation}
    A^{(h)}_{r,s}>0
    \qquad\text{for all }r,s=1,\ldots,h.
\end{equation}
Thus $A^{(h)}$ is a strictly positive symmetric matrix.  Perron--Frobenius implies that its leading eigenvalue is simple and that the corresponding eigenvector can be chosen with strictly positive coordinates in the basis $(\psi_{u,1},\ldots,\psi_{u,h})$.  This proves item 2.

\paragraph{Item 3: split blocks are nested and satisfy interlacing.}
The formula in \cref{eq:app_closed_split_form} contains no dependence on $h$ except through the allowed index range $1\le r,s\le h$.  Therefore
\begin{equation}
\label{eq:app_split_principal_submatrix}
    A^{(h)}=A^{(h+1)}_{1:h,\,1:h}.
\end{equation}
Thus $A^{(h)}$ is a leading principal submatrix of $A^{(h+1)}$.  Since these matrices are symmetric, Cauchy's interlacing theorem gives
\begin{equation}
    \lambda^{(h+1)}_1
    \ge
    \lambda^{(h)}_1
    \ge
    \lambda^{(h+1)}_2
    \ge
    \cdots
    \ge
    \lambda^{(h+1)}_h
    \ge
    \lambda^{(h)}_h
    \ge
    \lambda^{(h+1)}_{h+1}.
\end{equation}
This proves item 3.

\paragraph{Item 4: leading split eigenvalues are ordered coarse-to-fine.}
Taking the first inequality from item 3 for $h=1,2,\ldots,L-1$ gives
\begin{equation}
    \lambda^{(L)}_1
    \ge
    \lambda^{(L-1)}_1
    \ge
    \cdots
    \ge
    \lambda^{(1)}_1.
\end{equation}
A split block of height $h$ corresponds to a subtree of height $h$; larger $h$ means a coarser split higher in the hierarchy.  Hence the dominant split modes are ordered from coarser to finer subtrees, proving item 4.

\subsection{Exponential kernel used in the parametric model}
\label[app]{app:exp_kernel_proof}

The qualitative theory above only requires $f$ to be positive and decreasing.  The parametric model in \cref{sec:parametric_decay} uses
\begin{equation}
\label{eq:app_exp_kernel}
    f(d)=\alpha e^{-\beta d},
    \qquad \alpha,\beta>0.
\end{equation}
This subsection records the additional algebraic structure of that special case.

\begin{proposition}[Rank-one relation for the exponential kernel]
\label{prop:exp_rank_one}
Let $f$ be given by \cref{eq:app_exp_kernel}.  Define
\begin{equation}
    q_r:=2^{r/2}e^{-\beta r},
    \qquad
    \theta:=\frac{e^{2\beta}}{2},
    \qquad
    q_0:=1.
\end{equation}
For a height-$h$ subtree, the local scaling block and split block satisfy
\begin{equation}
\label{eq:app_exp_rank_one_relation}
    A^{\mathrm{sc},(h)}
    =
    \begin{pmatrix}
        0 & 0\\
        0 & A^{(h)}
    \end{pmatrix}
    +
    \alpha
    \begin{pmatrix}
        1\\ q
    \end{pmatrix}
    \begin{pmatrix}
        1\\ q
    \end{pmatrix}^{\!\top},
    \qquad
    q=(q_1,\ldots,q_h)^\top.
\end{equation}
Consequently, for the full tree,
\begin{equation}
\label{eq:app_rank_one_interlacing_global}
    \lambda^{\mathrm{sc}}_1
    \ge
    \lambda^{(L)}_1
    \ge
    \lambda^{\mathrm{sc}}_2
    \ge
    \lambda^{(L)}_2
    \ge
    \lambda^{\mathrm{sc}}_3
    \ge
    \cdots
    \ge
    \lambda^{(L)}_L
    \ge
    \lambda^{\mathrm{sc}}_{L+1}.
\end{equation}
\end{proposition}

\begin{proof}
Substituting $f(d)=\alpha e^{-\beta d}$ into \cref{eq:app_closed_scaling_form,eq:app_closed_split_form} gives, with $m=\min(r,s)$,
\begin{equation}
\label{eq:app_exp_scaling_compact}
    A^{\mathrm{sc},(h)}_{r,s}
    =
    \alpha q_rq_s
    \left[
        \theta^m
        +
        \frac12\sum_{a=0}^{m-1}\theta^a
    \right],
    \qquad r,s=0,\ldots,h,
\end{equation}
and
\begin{equation}
\label{eq:app_exp_split_compact}
    A^{(h)}_{r,s}
    =
    \alpha q_rq_s
    \left[
        \theta^m
        +
        \frac12\sum_{a=1}^{m-1}\theta^a
        -
        \frac12
    \right],
    \qquad r,s=1,\ldots,h.
\end{equation}
Comparing the two displays yields
\begin{equation}
    A^{\mathrm{sc},(h)}_{0,0}=\alpha,
    \qquad
    A^{\mathrm{sc},(h)}_{0,r}=\alpha q_r,
    \qquad
    A^{\mathrm{sc},(h)}_{r,s}=A^{(h)}_{r,s}+\alpha q_rq_s
    \quad r,s\ge 1.
\end{equation}
This is exactly \cref{eq:app_exp_rank_one_relation}.

Taking \(h=L\), the full scaling block is a positive rank-one perturbation of \(0 \oplus A^{(L)}\). For the exponential kernel, \(A^{(L)}\) is positive semidefinite because \(e^{-\beta \mathrm{dist}(i,j)}\) is a positive-semidefinite tree kernel; hence the eigenvalues of \(0\oplus A^{(L)}\) are ordered as \(\lambda^{(L)}_1\ge\cdots\ge\lambda^{(L)}_L\ge 0\). Cauchy's interlacing theorem for a rank-one symmetric perturbation then gives \cref{eq:app_rank_one_interlacing_global}.
\end{proof}

\begin{corollary}[Placement of non-leading scaling modes]
\label{cor:scaling_mode_placement}
Under the exponential kernel, for $k=1,\ldots,L-1$,
\begin{equation}
    \lambda^{(L-k)}_1
    \ge
    \lambda^{(L)}_{k+1}
    \ge
    \lambda^{\mathrm{sc}}_{k+2}.
\end{equation}
In particular,
\begin{equation}
    \lambda^{(L-1)}_1\ge \lambda^{\mathrm{sc}}_3,
    \qquad
    \lambda^{(L-2)}_1\ge \lambda^{\mathrm{sc}}_4,
    \qquad\text{and so on.}
\end{equation}
\end{corollary}

\begin{proof}
The inequality $\lambda^{(L)}_{k+1}\ge \lambda^{\mathrm{sc}}_{k+2}$ follows from \cref{eq:app_rank_one_interlacing_global}.  The inequality $\lambda^{(L-k)}_1\ge \lambda^{(L)}_{k+1}$ follows by iterating the split-block interlacing inequalities from \cref{thm:ordering}.  Combining the two inequalities gives the result.
\end{proof}

\subsection{Perturbation statement used in the noise discussion}
\label[app]{app:perturbation_comment}

The main text uses the standard perturbative intuition that small empirical fluctuations lift degeneracies and rotate directions inside nearly degenerate eigenspaces, while preserving well-separated spectral subspaces.  The precise statement needed for our use of top-$k$ eigenspace alignment is the following standard consequence of the Davis--Kahan sin-$\Theta$ theorem.

\begin{proposition}[Stability of clustered eigenspaces]
\label{prop:davis_kahan}
Let $M_0$ be the ideal distance-kernel matrix and let $M=M_0+E$ be a symmetric perturbation.  Let $U$ span an eigenspace, or a cluster of eigenspaces, of $M_0$ whose eigenvalues are separated from the rest of the spectrum by a gap $\gamma>0$.  Let $\widehat U$ be the corresponding invariant subspace of $M$.  If $\|E\|_2<\gamma$, then
\begin{equation}
    \|\sin\Theta(U,\widehat U)\|_2
    \le
    \frac{\|E\|_2}{\gamma}.
\end{equation}
\end{proposition}

\begin{proof}
This is the Davis--Kahan subspace perturbation bound applied to the symmetric pair $(M_0,M_0+E)$.  It implies that individual eigenvectors may rotate substantially within an exactly or nearly degenerate cluster, but the invariant subspace associated with the cluster is stable when the perturbation is small relative to the spectral gap.  This is why the empirical comparisons in \cref{sec:alignment_metric} use top-$k$ subspace alignment rather than individual eigenvector alignment.
\end{proof}

\section{Extension of theory to more general trees}
\label[app]{app:generalize_trees}
\subsection{Spherically symmetric trees}
\label[app]{app:spherically_symmetric_trees}

Let the tree have depth \(L\), and suppose that every node at depth \(d\) has \(b_d\) children, for \(d=0,\dots,L-1\). For a node \(u\) at depth \(d\), define
\begin{equation}
    N_{d,r}
    :=
    \prod_{q=d}^{d+r-1} b_q,
    \qquad
    N_{d,0}:=1,
\end{equation}
so that \(N_{d,r}=|T_r(u)|\) is the number of descendants of \(u\) at relative depth \(r\).

As in the binary case, define the scaling modes by
\begin{equation}
    \phi_\ell
    =
    \frac{1}{\sqrt{N_{0,\ell}}}
    \sum_{i\in T_\ell(0)} e_i,
    \qquad
    \ell=0,\dots,L,
\end{equation}
and set
\begin{equation}
    S_{\mathrm{sc}}
    :=
    \operatorname{span}\{\phi_0,\dots,\phi_L\}.
\end{equation}

For an internal node \(u\) at depth \(d\), let \(u_1,\dots,u_{b_d}\) denote its children, and choose an orthonormal basis
\begin{equation}
    c^{(1)},\dots,c^{(b_d-1)}
    \in
    \left\{
        c\in\mathbb R^{b_d}:\sum_{q=1}^{b_d} c_q=0
    \right\}.
\end{equation}
For \(a=1,\dots,b_d-1\) and \(r=1,\dots,L-d\), define
\begin{equation}
    \psi_{u,r,a}
    =
    \frac{1}{\sqrt{N_{d+1,r-1}}}
    \sum_{q=1}^{b_d} c^{(a)}_q
    \sum_{i\in T_{r-1}(u_q)} e_i .
\end{equation}
These modes contrast the child subtrees of \(u\) while remaining constant on relative-depth layers inside each child subtree. Define
\begin{equation}
    S^u
    :=
    \operatorname{span}
    \{
        \psi_{u,r,a}:
        r=1,\dots,L-d,\ a=1,\dots,b_d-1
    \}.
\end{equation}

\begin{theorem}[Hierarchy-adapted decomposition for spherically symmetric trees]
\label{thm:spherical_tree_decomposition}
The modes \(\{\phi_\ell\}\) and \(\{\psi_{u,r,a}\}\) form an orthonormal basis of the vertex space, giving
\begin{equation}
    \mathbb R^{|T|}
    =
    S_{\mathrm{sc}}
    \oplus
    \bigoplus_{u\in \mathrm{Int}(T)} S^u .
\end{equation}
Moreover, if
\begin{equation}
    M_{ij}=f(\mathrm{dist}(i,j)),
\end{equation}
then \(M\) preserves \(S_{\mathrm{sc}}\) and every \(S^u\). In the corresponding hierarchy-adapted basis,
\begin{equation}
    M
    \cong
    A^{\mathrm{sc}}
    \oplus
    \bigoplus_{d=0}^{L-1}
    \bigoplus_{u:\, |u|=d}
    \left(A^{[d]}\otimes I_{b_d-1}\right),
\end{equation}
where \(A^{[d]}\in\mathbb R^{(L-d)\times(L-d)}\) acts on the relative-depth coordinate and \(I_{b_d-1}\) acts on the child-contrast coordinate.
\end{theorem}

\begin{proof}
The proof follows the same orthogonality and cancellation argument as in the binary case. Scaling modes at different depths have disjoint support. Each split mode has zero sum across the children of its root, so it is orthogonal to every scaling mode. Two split modes are orthogonal if their supports are disjoint or lie on different depth layers. If one split root is an ancestor of the other, then the ancestor split mode is constant on the descendant subtree while the descendant split mode has zero child-sum, so their inner product vanishes. Counting the modes gives
\begin{equation}
    (L+1)
    +
    \sum_{d=0}^{L-1}
    \left(\prod_{q=0}^{d-1} b_q\right)(L-d)(b_d-1)
    =
    |T|,
\end{equation}
so the orthonormal collection is a basis.

Now let \(M_{ij}=f(\mathrm{dist}(i,j))\). Since distance is invariant under every rooted-tree automorphism, and in particular under permutations of sibling subtrees at any internal node, \(M\) sends depth-layer-constant vectors to depth-layer-constant vectors. Hence \(S_{\mathrm{sc}}\) is invariant.

For a split space \(S^u\), contributions from the child subtrees of \(u\) cancel outside \(T(u)\), because the child-contrast coefficients sum to zero and paired nodes have equal distance to any point outside \(T(u)\). Inside \(T(u)\), symmetry implies that the result is still supported on \(T(u)\), lies in the child-contrast subspace at \(u\), and can only mix the relative-depth coordinate and the contrast coordinate. Thus \(S^u\) is invariant.

Finally, for fixed depth \(d=|u|\), all nodes \(u\) have isomorphic descendant trees. The action of permuting the \(b_d\) children of \(u\) is the standard representation on the contrast index \(a\). Since \(M\) is invariant under these permutations, Schur's lemma implies that the restriction of \(M\) to \(S^u\) acts as the identity on the child-contrast factor and only mixes the relative-depth coordinate. Therefore the block has the form \(A^{[d]}\otimes I_{b_d-1}\).
\end{proof}

\paragraph{Irregular trees}
The exact invariant decomposition above uses spherical symmetry. For a fully irregular rooted tree, local Haar-like contrasts can still be defined, but a distance kernel \(M_{ij}=f(\mathrm{dist}(i,j))\) need not preserve each local split space exactly, because sibling subtrees may have different sizes or shapes. Thus irregular WordNet-induced trees should be understood as empirical perturbations of the symmetric idealization.

\subsection{\(s\)-ary trees}
\label[app]{app:sary_trees}

The full \(s\)-ary tree is the special case \(b_0=\cdots=b_{L-1}=s\). In this setting the depth-dependent blocks \(A^{[d]}\) depend only on the remaining subtree height \(h=L-d\), so we write them as
\(A^{(h)}\).

\begin{corollary}[Fixed \(s\)-ary trees]
\label{cor:sary_tree_extension}
Suppose \(b_0=\cdots=b_{L-1}=s\). Then
\begin{equation}
    M
    \cong
    A^{\mathrm{sc}}
    \oplus
    \bigoplus_{u\in \mathrm{Int}(T)}
    \left(A^{(h(u))}\otimes I_{s-1}\right).
\end{equation}
Moreover,
\begin{equation}
    A^{(h)}
    =
    A^{(h+1)}_{1:h,\,1:h}.
\end{equation}
Therefore Cauchy interlacing gives
\begin{equation}
    \lambda^{(h+1)}_1
    \ge
    \lambda^{(h)}_1
    \ge
    \lambda^{(h+1)}_2
    \ge
    \cdots
    \ge
    \lambda^{(h+1)}_h
    \ge
    \lambda^{(h)}_h
    \ge
    \lambda^{(h+1)}_{h+1}.
\end{equation}
In particular,
\begin{equation}
    \lambda^{(L)}_1
    \ge
    \lambda^{(L-1)}_1
    \ge
    \cdots
    \ge
    \lambda^{(1)}_1.
\end{equation}
Thus the dominant split modes are ordered from coarser to finer subtrees for every fixed branching factor \(s\ge 2\). The binary case in the main text is the special case \(s=2\).
\end{corollary}

\begin{proof}
When \(b_0=\cdots=b_{L-1}=s\), every height-\(h\) subtree is isomorphic to every other height-\(h\) subtree. Hence the split block depends only on \(h\). The height-\(h\) block is obtained from the height-\((h+1)\) block by restricting the relative-depth indices to \(1,\dots,h\), so it is a leading principal submatrix. The interlacing and monotonicity statements then follow from Cauchy's
interlacing theorem.
\end{proof}

\paragraph{Explicit block entries.}

We now record the explicit block entries for the fixed-branching case. To avoid overloading notation, let \(b\) denote the fixed branching factor of the full \(b\)-ary tree. For a node \(u\), recall that
\begin{equation}
    T_r(u)
    :=
    \{i\in T(u): i \text{ is at relative depth } r \text{ below } u\},
\end{equation}
so that \(|T_r(u)|=b^r\).

The local scaling block for a height-\(h\) subtree has entries
\begin{equation}
    A^{\mathrm{sc},(h)}_{r,t}
    =
    \frac{1}{\sqrt{b^r b^t}}
    \sum_{i\in T_r(u)}
    \sum_{j\in T_t(u)}
    f(\mathrm{dist}(i,j)),
    \qquad r,t=0,\dots,h.
\end{equation}
Grouping pairs by the relative depth \(\ell\) of their lowest common ancestor gives
\begin{equation}
    A^{\mathrm{sc},(h)}_{r,t}
    =
    b^{-(r+t)/2}
    \left[
        \sum_{\ell=0}^{\min(r,t)-1}
        (b-1)b^{r+t-\ell-1}
        f(r+t-2\ell)
        +
        b^{\max(r,t)}f(|r-t|)
    \right],
\end{equation}
where the empty sum is interpreted as zero. The first term counts pairs whose lowest common ancestor is strictly above both nodes; the final term counts pairs for which one node is ancestral to the other, including \(i=j\).

For split modes, fix an internal node \(u\). Let \(u_1,\dots,u_b\) be its children, and choose an orthonormal basis
\begin{equation}
    c^{(1)},\dots,c^{(b-1)}
    \in
    \left\{c\in\mathbb R^b:\sum_{q=1}^b c_q=0\right\}.
\end{equation}
For contrast direction \(a=1,\dots,b-1\), define
\begin{equation}
    \psi_{u,r,a}
    =
    \frac{1}{\sqrt{b^{r-1}}}
    \sum_{q=1}^b c^{(a)}_q
    \sum_{i\in T_{r-1}(u_q)} e_i,
    \qquad r=1,\dots,h(u).
\end{equation}
The split space at \(u\) is
\begin{equation}
    S^u
    =
    \operatorname{span}\{\psi_{u,r,a}: r=1,\dots,h(u),\ a=1,\dots,b-1\},
\end{equation}
so \(\dim S^u=h(u)(b-1)\).

Because \(M_{ij}=f(\mathrm{dist}(i,j))\) is invariant under permutations of the \(b\) child subtrees, Schur's lemma implies that \(M\) acts as the identity on the contrast index \(a\). Thus the full split block on \(S^u\) has the form
\begin{equation}
    A^{(h)}\otimes I_{b-1},
\end{equation}
up to the ordering of the basis. Equivalently,
\begin{equation}
    \psi_{u,r,a}^{\top}M\psi_{u,t,a'}
    =
    A^{(h)}_{r,t}\delta_{a,a'}.
\end{equation}
The depth block \(A^{(h)}\in\mathbb R^{h\times h}\) has entries
\begin{equation}
    A^{(h)}_{r,t}
    =
    \psi_{u,r,a}^{\top}M\psi_{u,t,a},
    \qquad r,t=1,\dots,h,
\end{equation}
for any fixed contrast direction \(a\). This quantity is independent of \(a\).

Grouping pairs by the relative depth \(\ell\) of their lowest common ancestor gives
\begin{equation}
\begin{split}
    A^{(h)}_{r,t} 
    &= b^{-(r+t-2)/2} \Bigg[ -b^{r+t-2}f(r+t) \\
    &\qquad + \sum_{\ell=1}^{\min(r,t)-1} (b-1)b^{r+t-\ell-2} f(r+t-2\ell) 
    + b^{\max(r,t)-1}f(|r-t|) \Bigg].
\end{split}
\end{equation}
The negative term comes from pairs lying in different child subtrees of \(u\); the positive terms come from pairs lying in the same child subtree. For \(b=2\), \(\dim S^u=h(u)\) and this reduces to the
binary split block used in the main text.

\paragraph{Exponential kernel.}

For the exponential kernel
\begin{equation}
    f(d)=\alpha e^{-\beta d},
    \qquad \alpha,\beta>0,
\end{equation}
the \(s\)-ary block formulas take a particularly compact form. Let
\begin{equation}
    x:=e^{-\beta},
    \qquad
    q_r:=s^{r/2}x^r,
    \qquad
    \theta:=\frac{1}{sx^2}.
\end{equation}
Then, for \(m=\min(r,t)\),
\begin{equation}
    A^{\mathrm{sc},(h)}_{r,t}
    =
    \alpha q_rq_t
    \left[
        \theta^m
        +
        \frac{s-1}{s}
        \sum_{a=0}^{m-1}\theta^a
    \right],
    \qquad r,t=0,\dots,h,
\end{equation}
and
\begin{equation}
    A^{(h)}_{r,t}
    =
    \alpha q_rq_t
    \left[
        \theta^m
        +
        \frac{s-1}{s}
        \sum_{a=1}^{m-1}\theta^a
        -
        \frac{1}{s}
    \right],
    \qquad r,t=1,\dots,h.
\end{equation}
For \(s=2\), these reduce to the binary formulas above.

Most importantly, the same rank-one relation survives for every \(s\). Comparing the two displays,
we obtain
\begin{equation}
    A^{\mathrm{sc},(h)}_{0,0}=\alpha,
    \qquad
    A^{\mathrm{sc},(h)}_{0,r}=\alpha q_r,
    \qquad
    A^{\mathrm{sc},(h)}_{r,t}
    =
    A^{(h)}_{r,t}+\alpha q_rq_t
    \quad r,t\ge 1.
\end{equation}
Equivalently,
\begin{equation}
    A^{\mathrm{sc},(h)}
    =
    \begin{pmatrix}
        0 & 0\\
        0 & A^{(h)}
    \end{pmatrix}
    +
    \alpha
    \begin{pmatrix}
        1\\ q
    \end{pmatrix}
    \begin{pmatrix}
        1\\ q
    \end{pmatrix}^{\!\top},
    \qquad
    q=(q_1,\dots,q_h)^\top .
\end{equation}

Thus the same spectral consequences hold. The local scaling block is a positive rank-one perturbation of \(0\oplus A^{(h)}\), so its eigenvalues interlace with those of \(A^{(h)}\). Since the exponential tree kernel \(e^{-\beta \mathrm{dist}(i,j)}\) is positive semidefinite, the split block \(A^{(h)}\) is also positive semidefinite, and the eigenvalues of \(0\oplus A^{(h)}\) are ordered with the additional zero eigenvalue at the end. Combined with the principal-submatrix relation above, this again implies that dominant split modes are ordered from coarser to finer subtrees, while the non-leading scaling modes are interleaved with the split spectrum. In particular, the qualitative hierarchy-aligned spectral picture derived for binary trees is not tied to binary branching; it is a consequence of distance-dependent co-occurrence on a symmetric rooted tree.

\section{Experimental details}
\label[app]{app:experimental_details}

This appendix gives implementation details for the main-text experiments. We first describe the construction of each main-text figure, then describe the shared preprocessing, WordNet hierarchy, co-occurrence, word2vec, and Gemma-unembedding pipelines. The treatment of indefinite \(M^\star\) and its positive spectral component is deferred to \cref{app:m_star_generalize}.

\subsection{Main-text figure construction}
\label[app]{app:main_figures}

\subsubsection{\cref{fig:figure1}: illustrative organism taxonomy}
\label[app]{app:fig1_details}

\cref{fig:figure1} is illustrative and is not used as a separate quantitative result. The displayed taxonomy is chosen to make the predicted splitting geometry visually interpretable and comparable to the organism example of Park et al.~\cite{park2025geometrycategoricalhierarchicalconcepts}. We fix the top-level structure to \emph{organism} branching into \emph{plant} and \emph{animal}; among valid descendant trees, we select a tree using simple frequency and WordNet-structure heuristics, described in \cref{app:organism_taxonomy}. The quantitative evidence comes from the sampled binary-subtree experiments in \cref{fig:micro_evidence}.

Panel (b) plots the mean normalized co-occurrence statistic \(M^\star_{ij}\) as a function of edge distance within this taxonomy. Panel (c) diagonalizes the theoretical Gram matrix obtained from a distance kernel evaluated on the full binary-tree distance matrix.

\subsubsection{\cref{fig:global_decay}: global co-occurrence decay}
\label[app]{app:fig2_details}

\cref{fig:global_decay} estimates the relationship between semantic distance and normalized co-occurrence. We compute \(M^\star\) from the corpus co-occurrence probabilities described in \cref{app:cooccurrence_statistics}. For each distance \(d=0,\dots,6\), we estimate the mean of \(M^\star_{ij}\) by sampling \(5000\) eligible synset pairs at that distance.

Pairs are sampled by first drawing an eligible synset \(u\), then using breadth-first search to count and sample an eligible synset \(v\) from the distance-\(d\) frontier of \(u\). For \(d=0\), we take \(v=u\). For \(d>0\), samples with empty eligible frontiers are skipped. To correct for the fact that the procedure samples a start node and then a frontier element, each sampled pair is weighted by the size of the eligible distance-\(d\) frontier from which \(v\) was drawn. Means and standard errors are computed from these weighted samples, using the Kish effective sample size \citep{kish1965survey}.

The left panel uses minimum undirected distance in the original WordNet noun hypernym DAG, restricted to eligible synsets. The right panel uses undirected edge distance in the contracted WordNet hierarchy described in \cref{app:wordnet_hierarchy}. Error bars indicate one standard error of the weighted binned mean.

We fit the exponential kernel
\[
    f(d)=\alpha e^{-\beta d}
\]
by weighted log-linear regression on the binned means. Specifically, for bins with positive finite mean \(\bar M_d\), we fit
\[
    \log \bar M_d = \log \alpha - \beta d
\]
using weights proportional to the inverse squared standard error of \(\bar M_d\). The contracted-hierarchy fit is used as the default theoretical kernel for the main empirical comparisons.

\subsubsection{\cref{fig:organism_eigvec}: organism example in theory, word2vec, and Gemma}
\label[app]{app:fig3_details}

\cref{fig:organism_eigvec} compares the theoretical, word2vec, and Gemma Gram matrices on the same illustrative organism taxonomy used in \cref{fig:figure1}. The theoretical Gram matrix is obtained by evaluating the fitted distance kernel on tree distances. The word2vec Gram matrix is \(WW^\top\), where \(W\) is constructed from the top positive eigenmodes of the restricted \(M^\star\) matrix according to \cref{eq:w2v_form}. The Gemma Gram matrix is computed from centered and whitened Gemma unembedding vectors restricted to the same tokens.

For the eigenvector visualizations, node colors represent the signed value of the corresponding eigenvector coordinate. The color scale runs from blue for negative values through white for values near zero to red for positive values. Since degenerate eigenspaces have no canonical basis, we display one representative direction from each degenerate eigenspace; such directions may appear as mixtures of same-level splits.

\subsubsection{\cref{fig:micro_evidence}: top-\(k\) eigenspace alignment}
\label[app]{app:fig4_details}

\cref{fig:micro_evidence} is the main quantitative experiment. We sample perfect binary subtrees of depth \(L=3\) from the contracted WordNet hierarchy. For each sampled tree, we compute a theoretical Gram matrix from the fitted distance kernel and empirical Gram matrices from word2vec and Gemma unembeddings.

For each empirical Gram matrix \(G_{\mathrm{emp}}\) and theoretical Gram matrix \(G_{\mathrm{th}}\), we compute
\[
    g(k)=\frac{1}{k}\|U_k^\top V_k\|_F^2,
\]
where \(U_k\) and \(V_k\) contain the top \(k\) eigenvectors of \(G_{\mathrm{emp}}\) and \(G_{\mathrm{th}}\), respectively. Eigenvectors are ordered by descending eigenvalue. The ambient dimension is \(15\), since a depth-\(3\) full binary tree has \(2^{4}-1=15\) nodes.

The shuffled-label baseline globally permutes the mapping from synsets to empirical vectors while keeping the theoretical tree fixed. This preserves the global distribution and spectrum of empirical vectors but destroys the semantic correspondence between positions in the hierarchy and representation vectors. For the root-level summary, we compute
\[
    \mathrm{AlignmentArea}
    =
    \sum_{k=1}^{15}
    \left(
    \frac{1}{k}\|U_k^\top V_k\|_F^2-\frac{k}{15}
    \right),
\]
where \(k/15\) is the expected overlap with a random \(k\)-dimensional subspace in a \(15\)-dimensional ambient space.

In addition to the global shuffled-label baseline, we also compute a within-tree shuffle baseline. 
For each sampled tree, let \(x_1,\ldots,x_n\) denote the empirical vectors assigned to the \(n=15\) tree positions in breadth-first order. 
The within-tree shuffle samples a random permutation \(\pi\) of \(\{1,\ldots,n\}\) and forms the permuted empirical Gram matrix
\[
    G^{\mathrm{within}}_{ij}
    =
    x_{\pi(i)}^\top x_{\pi(j)} .
\]
This baseline preserves the sampled tree, the multiset of empirical vectors appearing in that tree, the empirical Gram spectrum, and all pairwise similarities among those vectors, but destroys the assignment between particular semantic nodes and particular tree positions. 
It is therefore a more local control than the global shuffled-label baseline: the global shuffle tests whether the vocabulary-level semantic correspondence matters, while the within-tree shuffle tests whether the ordering of the vectors within each sampled hierarchy matters beyond the unordered vector content of that tree.

For every root and every representation setting, we sample up to \(5000\) valid binary subtrees, using all valid subtrees when fewer than \(5000\) are available. For each sampled tree, we compute one shuffled-label baseline by permuting the mapping from synsets to empirical vectors while keeping the theoretical tree fixed. Thus each root contributes up to \(5000\) original tree-level alignment curves and up to \(5000\) shuffled tree-level alignment curves.

\subsubsection{\cref{fig:Parkconsistency}: Park et al. concept-vector diagnostic}
\label[app]{app:fig5_details}

\cref{fig:Parkconsistency} applies the parent--child innovation diagnostic of Park et al.~\cite{park2025geometrycategoricalhierarchicalconcepts} to both Gemma unembeddings and theoretical co-occurrence embeddings. For each WordNet concept \(w\), Park et al. estimate a concept vector \(\bar\ell_w\) intended to represent membership in \(w\) or one of its descendants. For an immediate parent--child pair \((p,w)\), the diagnostic plots
\[
    \cos(\bar\ell_w-\bar\ell_p,\bar\ell_p).
\]
Under the parent-plus-orthogonal-innovation prediction, this quantity should concentrate near zero.

Following Park et al., concept vectors are estimated using a training subset consisting of \(70\%\) of descendant tokens independently selected for each synset. We apply the same estimator to Gemma unembeddings and to theoretical embeddings constructed from the fitted co-occurrence kernel. The shuffled-parent baseline replaces each true parent with a randomly selected eligible concept, breaking the WordNet parent--child relation while preserving the collection of concept vectors.

\subsection{Vocabulary preprocessing}
\label[app]{app:shared_preprocessing}

\subsubsection{Token matching}
\label[app]{app:token_matching}

We begin with WordNet noun lemmas that have co-occurrence statistics. We enforce token hygiene by lowercasing lemma strings, excluding digits, and excluding multiword or hyphenated lemmas unless otherwise stated. We then require that each lemma match a single Gemma tokenizer entry after removing tokenizer boundary markers, stripping whitespace, and lowercasing.

When multiple tokenizer entries match the same cleaned lemma, we prefer the exact leading-whitespace token, then other leading-boundary tokens, then shorter deterministic matches; lowercased surfaces are preferred to all-caps variants. After intersecting WordNet noun lemmas with the co-occurrence vocabulary and Gemma single-token vocabulary, \(17{,}566\) lemmas remain.

\subsubsection{Polysemy filtering}
\label[app]{app:polysemy_filter}

To obtain an unambiguous synset--lemma assignment, we apply a polysemy filter. For a lemma \(\ell\) and synset \(s\), define
\begin{equation}
    \mathrm{monosemy}(\ell,s)
    =
    \frac{\mathrm{count}(\ell,s)}
    {\sum_{s'}\mathrm{count}(\ell,s')},
\end{equation}

where counts are WordNet lemma counts summed across all parts of speech. If the total count is zero or unavailable, the score is set to zero. For each synset, we choose the eligible lemma with highest monosemy score, breaking ties by English Zipf frequency from \texttt{wordfreq}. Synsets for which no eligible lemma is uniquely preferred are removed. This produces \(11{,}735\) unique synset--lemma pairs.

\subsection{WordNet hierarchy construction}
\label[app]{app:wordnet_hierarchy}

WordNet's noun hypernym graph is a DAG because some synsets have multiple hypernyms. We convert it into a rooted arborescence before sampling binary subtrees. Starting from the global WordNet noun root, we assign each synset a single parent by choosing a hypernym parent that maximizes depth from the root. Ties are broken deterministically by sorting candidate parent synset names lexicographically and taking the first candidate.

We then restrict this arborescence to the eligible synsets from \cref{app:polysemy_filter}. Ineligible intermediate nodes are contracted out: the parent of an eligible synset becomes its nearest eligible ancestor in the rooted arborescence. The resulting contracted arborescence is the hierarchy used for the right panel of \cref{fig:global_decay}, the fitted theoretical kernel, and the binary-subtree experiments. Unless otherwise stated, semantic distance means undirected edge distance in this contracted hierarchy.

\subsection{Binary subtree sampling}
\label[app]{app:binary_subtree_sampling}

A valid depth-\(L\) binary subtree rooted at \(u\) is defined recursively. For \(L=0\), the only valid tree is \(u\). For \(L>0\), a valid tree consists of \(u\) and two distinct valid child subtrees of depth \(L-1\). We use \(L=3\) in the main quantitative experiments.

To sample uniformly over valid tree structures, we compute dynamic-programming counts \(N(u,L)\), the number of valid depth-\(L\) binary subtrees rooted at \(u\). At each internal node, child pairs are selected with probability proportional to the number of valid completions below them. This samples uniformly over admissible binary tree structures rather than over individual word sets.
\subsection{Co-occurrence statistics}
\label[app]{app:cooccurrence_statistics}

We construct co-occurrence statistics from the November 2023 English Wikipedia dump accessed through \url{https://huggingface.co/datasets/wikimedia/wikipedia}. The corpus is lowercased and tokenized by extracting
contiguous alphabetic spans using the regular expression \texttt{[a-z]+}. We discard articles
with fewer than \(500\) retained tokens.

The co-occurrence vocabulary consists of observed clean WordNet noun lemmas. Concretely, we
enumerate English WordNet noun lemmas, retain only lemmas matching \texttt{[a-z]+}, and
intersect this set with tokens observed in the preprocessed Wikipedia corpus. We then write the
filtered corpus as a memory-mapped token stream, removing tokens outside this vocabulary and
storing explicit article-boundary offsets so that co-occurrence pairs are never counted across
articles.

Given corpus tokens \(C[\nu]\), we estimate a symmetric weighted skip-gram co-occurrence
distribution with context window \(L=16\). For a token pair separated by distance \(d\), we use
weight \(r(d)=L/d\). The unnormalized co-occurrence count is
\begin{equation}
    \#(i,j)
=
\sum_{\nu}
\sum_{d=1}^{L}
\frac{L}{d}
\left[
\mathbf{1}\{C[\nu]=i,C[\nu+d]=j\}
+
\mathbf{1}\{C[\nu]=j,C[\nu+d]=i\}
\right],
\end{equation}

where the sums are restricted to positions within the same article. We normalize by the total
co-occurrence mass,
\begin{equation}
    P_{ij}
    =
    \frac{\#(i,j)}{\sum_{i',j'}\#(i',j')}.
\end{equation}
The unigram probabilities used in \(M^\star\) are the marginals of this co-occurrence
distribution,
\[
    P_i=\sum_j P_{ij}.
\]
Finally, we form the normalized co-occurrence statistic as in \cref{eq:mstar}.
All matrices are stored sparsely on the observed co-occurrence support. The relationship between
this possibly indefinite matrix and the positive semidefinite Gram matrix realized by Euclidean
embeddings is discussed in \cref{app:m_star_generalize}.

\subsection{word2vec-style embeddings}
\label[app]{app:w2v_construction}

We construct word2vec-style embeddings directly from the normalized co-occurrence
matrix. Let
\[
    M^\star_S = U\Lambda U^\top
\]
be the eigendecomposition of \(M^\star\) restricted to the eligible synset--lemma set after the
token-matching and polysemy filters of \cref{app:token_matching,app:polysemy_filter}. Since a
single Euclidean embedding space realizes a positive semidefinite Gram matrix, we retain the
positive spectral modes of \(M^\star_S\). In the main experiments, we keep the top \(d=2048\)
positive eigenmodes, matching the dimension of the Gemma unembedding vectors. If \(U_d\)
contains the retained eigenvectors and \(\Lambda_d\) the corresponding positive eigenvalues, we
define
\[
    W = U_d\Lambda_d^{1/2}.
\]
The word2vec-style Gram matrix used in the figures is \(WW^\top\). The treatment of
negative eigenvalues of \(M^\star\), and the corresponding PSD component of the co-occurrence
matrix, is given in \cref{app:m_star_generalize}.

\subsection{Gemma unembeddings and LLM internal activations}
\label[app]{app:gemma_unembeddings}

We use the pretrained \texttt{google/gemma-2b} checkpoint of Gemma Team~\citep{gemmateam2024gemmaopenmodelsbased}, accessed through the Hugging Face \texttt{transformers} library. This checkpoint is a two-billion-parameter decoder-only language model trained on 3T tokens, with a vocabulary of \(256{,}128\) tokens and representation dimension \(2048\). We analyze its unembedding vectors and, in robustness checks, contextual residual-stream activations through \texttt{TransformerLens}~\citep{nanda2022transformerlens}.

For each eligible lemma, we identify its matched single token and extract the corresponding vector from the model's output embedding matrix. In implementation, this is the matrix returned by \texttt{model.get\_output\_embeddings().weight}. We then center and whiten the full vocabulary of unembedding vectors before restricting to WordNet tokens. Let \(v_i\) denote the unembedding vector for token \(i\). We compute the empirical full-vocabulary mean \(\mu\) and covariance \(\Sigma\), and apply
\[
    \tilde v_i = \Sigma^{-1/2}(v_i-\mu),
\]
where \(\Sigma^{-1/2}\) is computed by eigendecomposition of the empirical covariance. The whitening transform is fit once on the full Gemma vocabulary and then applied to the eligible WordNet tokens. Robustness checks using globally centered but unwhitened Gemma vectors are reported in \cref{app:robustness_alternative_embeddings}.

For any selected tree or synset set \(S\), the Gemma Gram matrix is
\[
    G^{\mathrm{Gemma}}_{ij} = \tilde v_i^\top \tilde v_j,
    \qquad i,j\in S.
\]

For the LLM internal-activation controls, we also extract residual-stream activations from Gemma and Llama using \texttt{TransformerLens}. For each WordNet synset, we convert the selected best lemma to a raw title-cased string with no prefix or suffix; for example, \texttt{animal.n.01}, \texttt{cat.n.01}, and a multiword lemma such as \texttt{sea\_turtle} are represented by the prompts ``Animal'', ``Cat'', and ``Sea Turtle'', respectively. We then extract the final-token activation from the post-block residual stream,
\[
    \mathrm{blocks}.\ell.\mathrm{hook\_resid\_post},
\]
where \(\ell=\lfloor n_{\mathrm{layers}}/2 \rfloor\) is the exact middle layer of the model. The extracted activation vectors are globally centered across the eligible synset vocabulary before computing Gram matrices and alignment scores. We do not tune over layers or prompts; the middle layer and raw lexical prompt are used as a simple, model-agnostic choice.

For the Gemma internal-activation control, the middle layer is \(\ell=9\). Its residual-stream dimension is \(2048\), and its vocabulary size is \(256{,}128\).

For the Llama internal-activation control, we use \texttt{meta-llama/Llama-3.2-1B}, a member of the Llama 3 family of autoregressive transformer language models~\citep{llamaherd}. In the \texttt{TransformerLens} configuration used here, this model has \(16\) transformer layers, residual-stream dimension \(2048\), and vocabulary size \(128{,}256\); hence the exact middle layer is \(\ell=8\). We apply the same extraction protocol as for Gemma: raw title-cased best-lemma prompts, final-token residual-stream activations, and global centering across the eligible synset vocabulary.

The purpose of these internal-activation controls is not to optimize a model-specific semantic layer, but to test whether the hierarchical alignment signal appears in a simple contextual representation of modern decoder-only language models.

\subsection{Computational resources}
\label[app]{app:compute_resources}

All experiments were run on an Apple M4 CPU machine with approximately \(20\)GB of available RAM and no discrete GPU. The most expensive preprocessing step was the construction of Wikipedia co-occurrence statistics, which took roughly \(48\) hours. The eigendecomposition step took approximately \(30\) minutes. All other figure-generation and alignment computations took at most \(30\) minutes each. Extracting the Gemma and Llama internal activations used in the robustness controls took approximately \(30\) minutes in total.

\subsection{Existing assets and licenses.}
\label[app]{app:licenses}
We use WordNet, the November 2023 English Wikipedia dump, pretrained Gemma/Llama model resources, TransformerLens, and standard scientific Python libraries. These assets are cited in the paper, and the accompanying anonymized code/supplement lists the versions, sources, and available license or terms-of-use information. We use these assets for non-commercial research analysis and do not redistribute pretrained model weights or the Wikipedia dump.

\subsection{Illustrative organism taxonomy}
\label[app]{app:organism_taxonomy}

The organism taxonomy in \cref{fig:figure1,fig:organism_eigvec} is used only for visualization. We fixed the top-level structure to match Park et al.~\cite{park2025geometrycategoricalhierarchicalconcepts}: \emph{organism} branches into \emph{plant} and \emph{animal}. Among valid descendant trees satisfying this upper structure, we selected a tree using a lexicographic heuristic that first avoids low-frequency lemmas. The quantitative results do not depend on this illustrative choice and instead use the uniformly sampled binary subtrees described in \cref{app:binary_subtree_sampling}.

\section{The normalized co-occurrence matrix \(M^\star\)}
\label[app]{app:m_star_generalize}

\subsection{Definition, relation to PMI, and connection to word embeddings}
\label[app]{app:mstar_pmi_background}

This appendix gives additional background on the normalized co-occurrence matrix \(M^\star\) used throughout the paper. The starting point is the classical connection between word embedding algorithms and co-occurrence matrix factorization. Levy and Goldberg~\citep{levy2014neural} showed that the skip-gram with negative sampling objective implicitly targets a shifted pointwise mutual information matrix: at optimum, word--context inner products approximate
\begin{equation}
    \operatorname{PMI}_{ij} - \log k,
\end{equation}
where \(k\) is the number of negative samples and
\begin{equation}
    \operatorname{PMI}_{ij}
    :=
    \log \frac{P_{ij}}{P_iP_j}.
\end{equation}
Thus, in an idealized low-rank symmetric matrix-factorization view, the leading spectral modes of the PMI matrix determine the geometry of the learned embeddings. This is the sense in which word embedding geometry can be studied through the eigenspaces of a co-occurrence-derived target matrix.

In the main text we use a closely related normalized co-occurrence matrix, \(M^\star\), which arises in the symmetric-embedding, quadratic-loss approximation to skip-gram analyzed by Karkada et al.~\citep{karkada2025closedform}. Let \(P_{ij}\) denote the probability that words \(i\) and \(j\) co-occur within the chosen context window, and let \(P_i\) and \(P_j\) denote the corresponding unigram probabilities. The matrix \(M^\star\) is defined by
\begin{equation}
\label{eq:mstar_app_def}
M^\star_{ij}
=
\frac{P_{ij}-P_iP_j}{\frac12(P_{ij}+P_iP_j)}
=
\frac{2(P_{ij}-P_iP_j)}{P_{ij}+P_iP_j}.
\end{equation}
Karkada et al. show that, under this quadratic symmetric formulation, training can be viewed as approximately solving the matrix factorization problem
\begin{equation}
\label{eq:mstar_factorization_app}
    \widehat W
    =
    \arg\min_W
    \left\|WW^\top - M^\star\right\|_F^2 .
\end{equation}
Consequently, when \(M^\star\) is positive semidefinite and the embedding dimension is large enough, the learned Gram matrix \(WW^\top\) recovers \(M^\star\); at finite dimension it recovers its leading positive spectral modes. Thus \(M^\star\), like PMI, is a co-occurrence-derived matrix target for word embedding algorithms, but it corresponds to a slightly different objective and has useful normalization properties.

To see its relation to PMI explicitly, write the co-occurrence ratio as
\begin{equation}
    R_{ij}:=\frac{P_{ij}}{P_iP_j}.
\end{equation}
Then
\begin{equation}
\label{eq:mstar_ratio_transform}
M^\star_{ij}
=
\frac{2(R_{ij}-1)}{R_{ij}+1}.
\end{equation}
Thus \(M^\star_{ij}\) is a monotone transform of the same pairwise co-occurrence signal measured by PMI. If \(x_{ij}:=\log R_{ij}=\operatorname{PMI}_{ij}\), then \(R_{ij}=e^{x_{ij}}\), so
\begin{equation}
\label{eq:mstar_tanh_pmi}
    M^\star_{ij}
    =
    2\frac{e^{x_{ij}}-1}{e^{x_{ij}}+1}
    =
    2\tanh\left(\frac{x_{ij}}{2}\right)
    =
    2\tanh\left(\frac{\operatorname{PMI}_{ij}}{2}\right).
\end{equation}
Consequently, when \(R_{ij}\) is close to \(1\), or equivalently when \(\operatorname{PMI}_{ij}\) is close to zero,
\begin{equation}
    M^\star_{ij}
    \approx
    \operatorname{PMI}_{ij}.
\end{equation}

The advantage of \(M^\star\) is that it preserves the ordering of pairwise co-occurrence strength while remaining bounded:
\begin{equation}
    -2 \le M^\star_{ij} \le 2.
\end{equation}
It is therefore a bounded, symmetric normalization of the PMI signal. In the experiments, we work with \(M^\star\) because it is the matrix target associated with the symmetric quadratic embedding objective in \cref{eq:mstar_factorization_app}, while retaining the same local interpretation as PMI as a normalized measure of excess co-occurrence.

\subsection{Positive spectral component and Euclidean embedding geometry}
\label[app]{app:mstar_psd_component}

The main text presents the cleanest version of the theory in the idealized case where \(M^\star\) is positive semidefinite. In that setting, a Euclidean embedding Gram matrix can be identified directly with \(M^\star\), and the hierarchy-aligned spectral theory applies to \(M^\star\) itself. Empirically, however, \(M^\star\) need not be positive semidefinite. This matters because a single Euclidean embedding space can only realize positive semidefinite Gram matrices.

Let
\begin{equation}
\label{eq:mstar_eigendecomp}
    M^\star = U\Lambda U^\top
\end{equation}
be the eigendecomposition of \(M^\star\). Define the positive and negative spectral components
\begin{equation}
\label{eq:m_plus_minus_def}
    M^+
    :=
    U\Lambda_+U^\top,
    \qquad
    M^-
    :=
    U\Lambda_-U^\top,
\end{equation}
where
\begin{equation}
\label{eq:lambda_plus_minus_def}
    (\Lambda_+)_{aa}:=\max(\Lambda_{aa},0),
    \qquad
    (\Lambda_-)_{aa}:=\max(-\Lambda_{aa},0).
\end{equation}
Then
\begin{equation}
\label{eq:mstar_psd_nsd_decomp}
    M^\star = M^+ - M^-,
\end{equation}
with both \(M^+\) and \(M^-\) positive semidefinite.

The positive component \(M^+\) is the part of \(M^\star\) that can be represented as a Euclidean Gram matrix. If all positive eigenmodes are retained, the corresponding Gram matrix is \(M^+\). If only the top \(r\) positive eigenmodes are retained, the realized rank-\(r\) Gram matrix is
\begin{equation}
\label{eq:mplus_rank_r}
    M^+_r
    =
    U_r\Lambda_r U_r^\top
    =
    W_r W_r^\top,
    \qquad
    W_r := U_r\Lambda_r^{1/2},
\end{equation}
where \(\Lambda_r\) contains the top \(r\) positive eigenvalues and \(U_r\) contains the corresponding eigenvectors. This is the finite-dimensional Euclidean geometry recovered by a single-space spectral embedding.

This motivates the following softened version of the distance-kernel assumption used in the main text.

\begin{assumption}[PSD hierarchical distance structure]
\label{ass:psd_distance_structure}
Let \(S\) be a selected set of WordNet nodes, and let \(\mathrm{dist}(i,j)\) denote their tree distance. The positive spectral component of the normalized co-occurrence matrix satisfies
\begin{equation}
\label{eq:m_plus_distance_assumption}
    M^+_{ij}
    =
    C^+(\mathrm{dist}(i,j)),
    \qquad i,j\in S,
\end{equation}
for some positive decreasing function \(C^+\).
\end{assumption}

This assumption is weaker than requiring \(M^\star\) itself to be both positive semidefinite and distance-dependent. The idealized assumption in the main text corresponds to the special case \(M^-=0\), so that \(M^\star=M^+\). More generally, \(M^\star\) may contain negative spectral components, but the single-space Euclidean embedding geometry is governed by \(M^+\). Therefore, the block-decomposition and hierarchy-aligned eigenspace arguments apply directly to \(M^+\) whenever \cref{ass:psd_distance_structure} holds.

Empirically, this softened assumption is supported by \cref{fig:m_plus_decay}. We restrict \(M^\star\) to the eligible lemma set, symmetrize the restricted matrix, compute its leading positive eigenmodes, and construct low-rank PSD truncations \(M^+_r\) for several ranks \(r\). For each rank, we estimate the mean entry of \(M^+_r\) as a function of WordNet distance by the same distance-bin sampling procedure used in \cref{fig:global_decay}. Across ranks, \(M^+_r\) decays monotonically with semantic distance, showing that the hierarchical distance signal is present in the Gram-realizable part of the co-occurrence matrix.

\begin{figure}[t]
    \centering
    \includegraphics[width=0.47\textwidth]{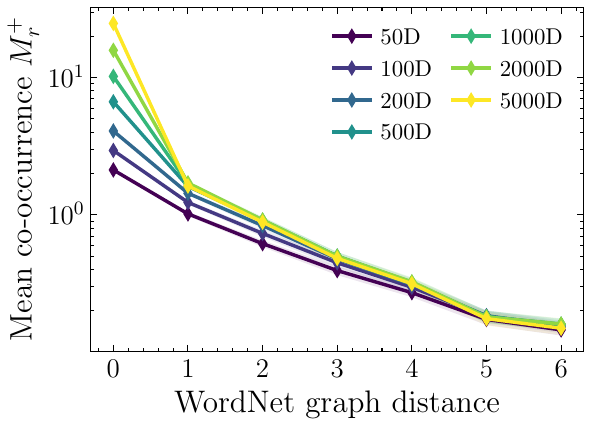}
    \hspace{0.03\textwidth}
    \includegraphics[width=0.47\textwidth]{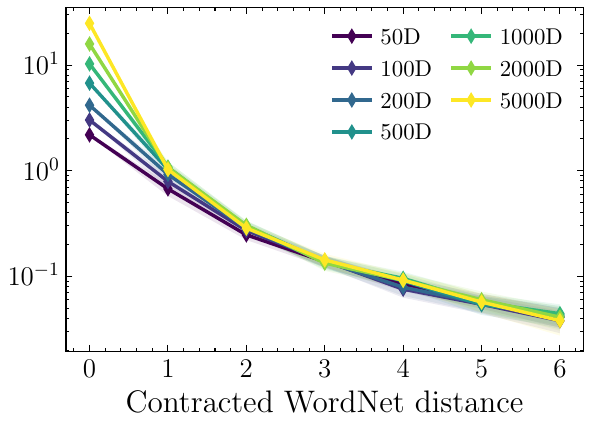}
    
    \caption{
    \textbf{Distance decay in low-rank PSD truncations of \(M^\star\).}
    We restrict \(M^\star\) to the eligible lemma set, retain the leading \(r\) positive eigenmodes, and form rank-\(r\) PSD Gram matrices \(M^+_r=U_r\Lambda_rU_r^\top\). Curves show the mean entry of \(M^+_r\) as a function of semantic distance in the original WordNet graph (left) and in the contracted arborescence used in our experiments (right). The decay persists across ranks and distance constructions, supporting the PSD hierarchical distance structure in \cref{ass:psd_distance_structure}.
    }
    \label{fig:m_plus_decay}
\end{figure}

In the main text, we fit the exponential distance kernel using the raw pairwise statistic \(M^\star\), because \(M^\star\) is computed directly from pairwise co-occurrence and unigram probabilities. By contrast, \(M^+\) is defined through a global spectral transformation of the full matrix. \Cref{fig:m_plus_decay} shows that this simplification does not drive the phenomenon: the same distance-decay structure appears in the PSD component that determines Euclidean embedding geometry.

\clearpage
\section{Additional robustness analyses}
\label[app]{app:further_evidence}

This appendix collects additional checks on the empirical assumptions and alignment results used in the main text. We begin with a diagnostic for the basic distance-kernel approximation \(M^\star_{ij}\approx f(D_{ij})\). A possible concern is that the apparent decay with induced tree distance \(D_{ij}\) could instead be explained by a coarser topological variable, such as the depth of the lowest common ancestor of \(i\) and \(j\). In that case, \(D_{ij}\) would not be the right one-dimensional summary of co-occurrence structure: pairs with small distance would appear to co-occur more often only because they tend to share deeper, more specific ancestors.

\Cref{fig:fixedDepthLCA} addresses this concern by conditioning on lowest-common-ancestor depth. Importantly, this diagnostic is computed over the same distribution of induced tree geometries used in the alignment experiments, rather than over uniformly random pairs from WordNet. For each sampled complete binary tree under a candidate root, we evaluate \(M^\star_{ij}\) for every unordered pair of nodes in the sampled tree, including diagonal pairs, and assign each pair its induced tree distance \(D_{ij}\) and induced \(\mathrm{depthLCA}_{ij}\). We then bin pairs jointly by \((D_{ij},\mathrm{depthLCA}_{ij})\), compute root-level bin means after aggregating over that root's sampled trees, and average these root-level means across roots. Within each fixed-\(\mathrm{depthLCA}\) stratum, the mean co-occurrence statistic still decreases with \(D_{ij}\). Thus the decay captured by \(f(D_{ij})\) is not merely a consequence of comparing pairs whose common ancestor lies at different depths.

Having checked this structural assumption, we next ask whether the reported eigenspace-alignment signal depends on other modeling or representation choices. The remaining robustness checks vary the null baseline, the parametric form of the distance kernel, the sampled tree depth, and the LLM representation used to form the empirical Gram matrix.

\begin{figure*}[t]
    \centering
    \includegraphics[width=0.85\textwidth]{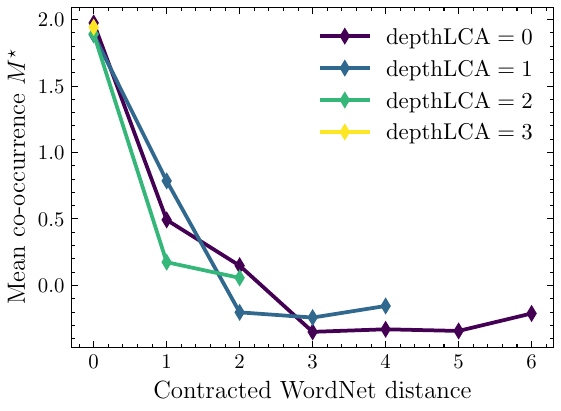}
    \caption{
    \textbf{Distance decay persists after conditioning on lowest-common-ancestor depth.}
    For each sampled complete binary tree, we evaluate \(M^\star_{ij}\) for every unordered node pair \((i,j)\), including diagonal pairs, and assign each pair its induced tree distance \(D_{ij}\) and lowest-common-ancestor depth \(\mathrm{depthLCA}_{ij}\). Pairs are then grouped by \((D_{ij}, \mathrm{depthLCA}_{ij})\), and bin means are computed after aggregating over sampled trees for each root and then averaged across roots. Colors indicate fixed \(\mathrm{depthLCA}\) strata. Within each stratum, the mean co-occurrence statistic decreases with \(D_{ij}\), showing that the observed decay is not explained solely by common-ancestor depth and supporting the use of a distance-based model \(f(D_{ij})\).
    }
    \label{fig:fixedDepthLCA}
\end{figure*}

\subsection{Additional baseline}

Before varying kernel form, tree depth, or representation construction, we first show that the main eigenspace-alignment result remains above a stricter within-tree shuffle control. This control is stricter than the global shuffled-label baseline because it preserves the exact empirical vectors appearing in each sampled tree and only destroys their assignment to tree positions. This comparison is shown in \cref{fig:figure4_within_tree_baseline}.

\begin{figure*}[t]
    \centering
    \includegraphics[width=1.\textwidth]{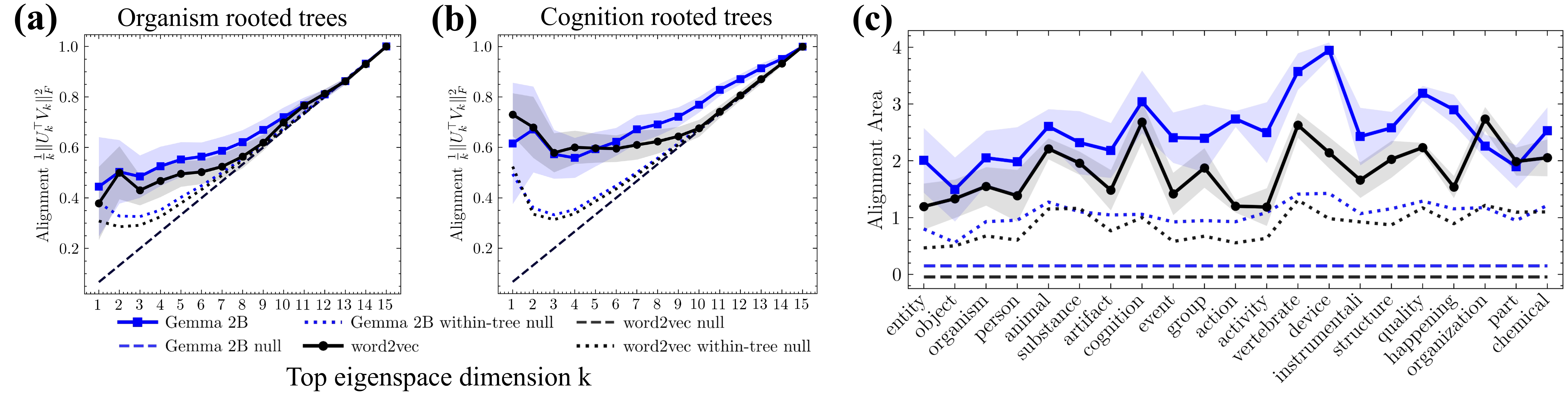}
    \caption{
    \textbf{Recreation of the main eigenspace-alignment experiment with an additional within-tree shuffle baseline.}
    We repeat the top-\(k\) eigenspace-alignment analysis from \cref{fig:micro_evidence} while displaying both null baselines. 
    Solid curves show the empirical alignment between the theoretical tree-kernel Gram matrix and the representation Gram matrix. 
    The global shuffled-label baseline randomly permutes the mapping from synsets to empirical vectors across the full eligible vocabulary, preserving the global vector distribution while destroying the semantic correspondence between WordNet nodes and representation vectors. 
    The within-tree shuffle baseline is computed separately for each sampled tree: it preserves the tree, the multiset of empirical vectors in that tree, and the empirical Gram spectrum, but randomly permutes those vectors across the tree positions before recomputing alignment to the same theoretical Gram matrix. 
    Thus, the within-tree baseline controls for unordered vector content within each sampled tree, whereas the global shuffle controls for vocabulary-level label correspondence. 
    The empirical alignment remains above both baselines, showing that the effect is not explained solely by the global distribution of vectors or by the unordered collection of vectors appearing in a sampled tree. 
    Shaded regions indicate one standard deviation across sampled trees.
    }
    \label{fig:figure4_within_tree_baseline}
\end{figure*}

\subsection{Kernel form}
We repeated the analysis using an alternative kernel form for the theoretical tree covariance. In addition to the exponential kernel used in the main analysis, we fit a shifted power-law kernel of the form
\[
    f(d) = \alpha (1+d)^{-\beta}.
\]
The resulting alignment curves remain above both the global and within-tree baselines, indicating that the qualitative effect is not specific to the exponential decay assumption. This control is shown in \cref{fig:kernel_robustness}.

\begin{figure*}[t]
    \centering
    \includegraphics[width=0.85\textwidth]{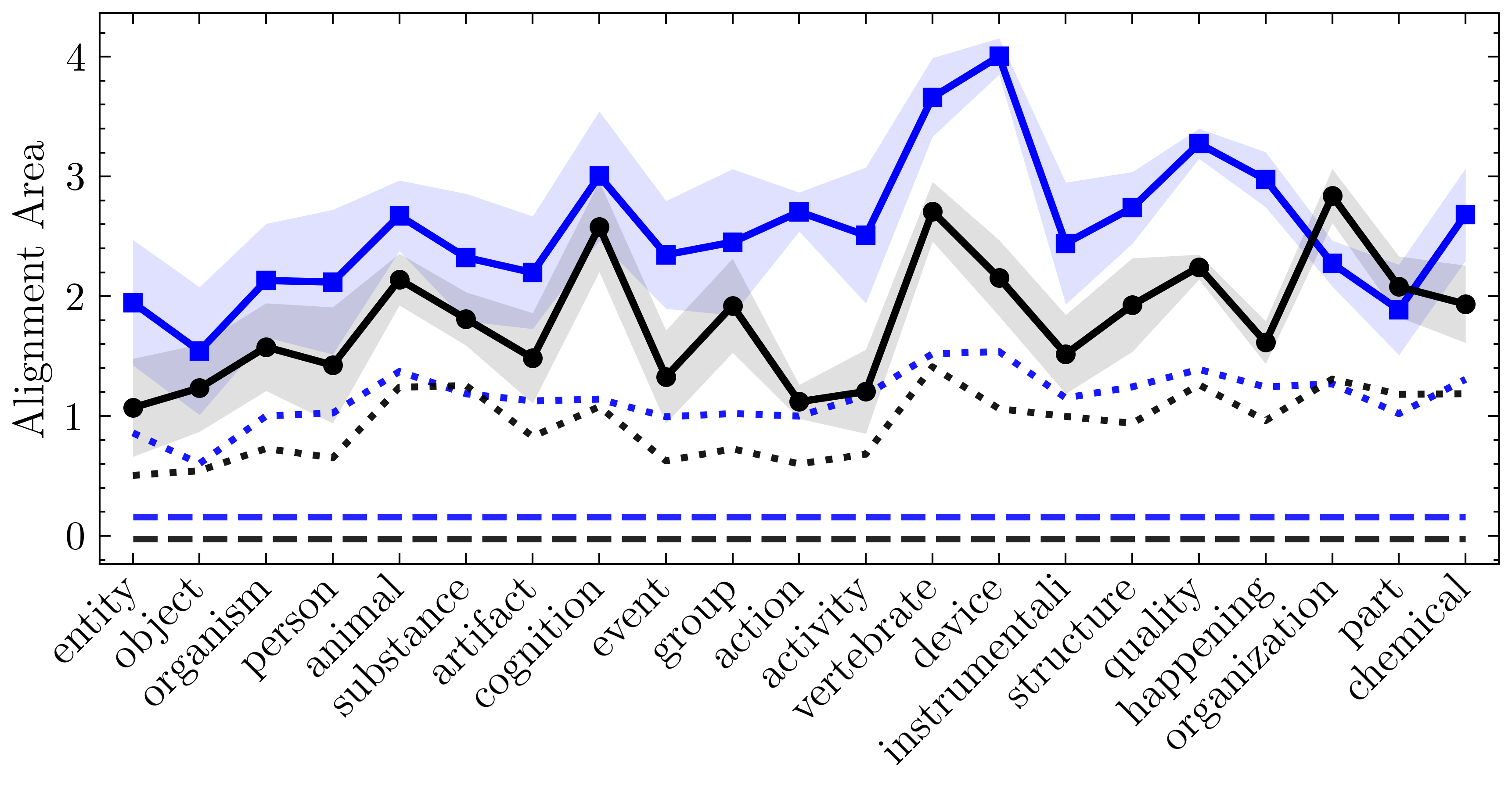}
    \includegraphics[width=0.85\textwidth]{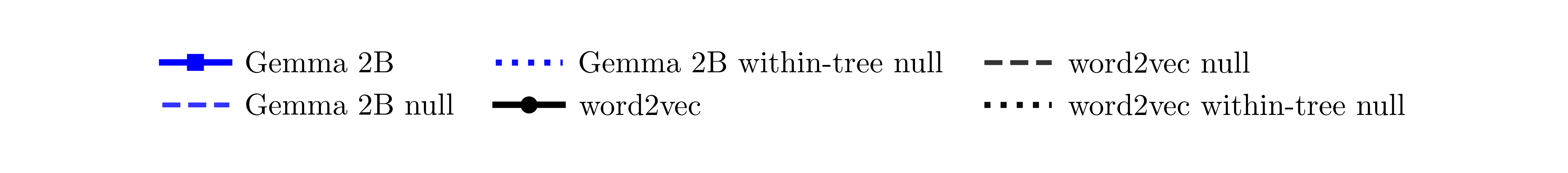}
    \caption{
    \textbf{Robustness to kernel parameterization.}
    We compare the alignment obtained using the exponential distance kernel from the main text with the alignment obtained using a shifted power-law kernel,
    \(f(d)=1.967\cdot (1+d)^{-2.153}\).
    The empirical alignment remains above both the global shuffle and within-tree shuffle baselines, indicating that the result is driven by distance-dependent hierarchical decay rather than by the specific exponential parameterization.
    }
    \label{fig:kernel_robustness}
\end{figure*}

\subsection{Smaller sampled trees}
\label[app]{app:l2_repeat}
Second, we repeated the root-sweep analysis with tree depth \(L=2\), corresponding to binary trees with \(2^{L+1}-1=7\) nodes rather than the \(15\)-node trees used when \(L=3\). This smaller tree size substantially increases the number of eligible roots, because fewer descendants are required to construct a complete sampled binary subtree. Despite this larger and less restrictive root set, the alignment area remains elevated above both baselines for many roots. Thus, the effect is not driven by a small subset of roots that happen to support larger complete trees. This control is shown in \cref{fig:l2_root_sweep}.

\begin{figure*}[t]
    \centering
    \includegraphics[width=0.85\textwidth]{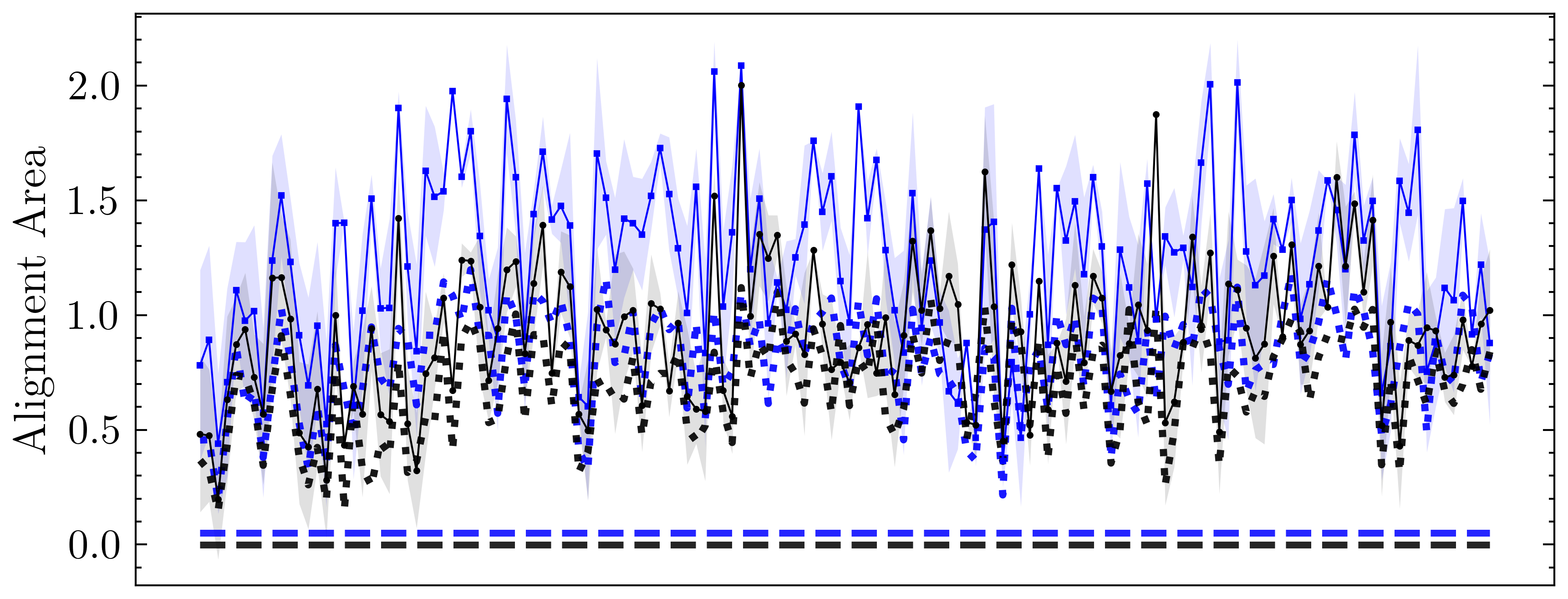}
    \includegraphics[width=0.85\textwidth]{imgs/appendix/area_legend.png}
    \caption{
    \textbf{Root-sweep robustness for \(L=2\).}
    We repeat the root-sweep analysis using complete binary trees of depth \(L=2\), corresponding to \(2^{L+1}-1=7\) nodes per sampled tree.
    This relaxation substantially increases the number of eligible roots relative to \(L=3\), from 21 to 144. The alignment area remains above the within-tree shuffle baseline for most roots.
    }
    \label{fig:l2_root_sweep}
\end{figure*}

\subsection{Alternative representations}
\label[app]{app:robustness_alternative_embeddings} 
Third, we evaluated alternative representations beyond the whitened Gemma embedding space used in the main analysis. Specifically, we repeated the alignment analysis using Gemma unembeddings, Gemma internal residual-stream activations, and Llama internal residual-stream activations. In each case, we subtracted the global mean vector from the representation before computing alignment. 

For the internal activation controls, we used the exact middle transformer layer in each model, chosen as \(\lfloor n_{\mathrm{layers}}/2 \rfloor\), and used the raw title-cased best lemma as the input string, e.g. ``Animal'', ``Cat'', or ``Sea Turtle''. This choice is intentionally simple and model-agnostic: it avoids tuning the layer index and avoids adding sentence-level context that could introduce additional uncontrolled structure.

All representation variants show alignment above the null baselines, although the alignment is much weaker for the internal representations. These representation controls are shown in \cref{fig:centered_internal_unwhitened_controls}.

\begin{figure*}[t]
    \centering
    \includegraphics[width=0.90\textwidth]{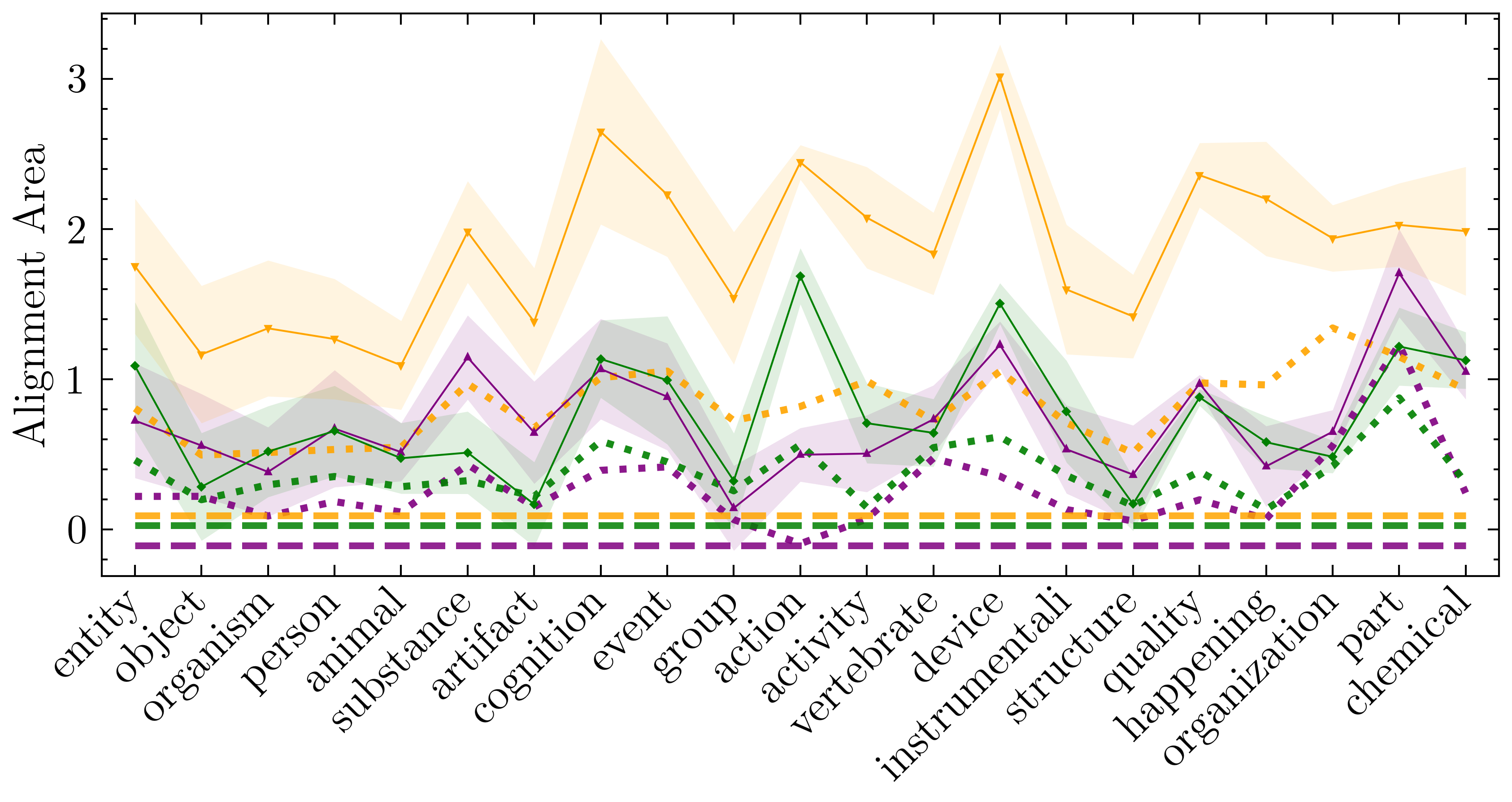}
    \includegraphics[width=0.90\textwidth]{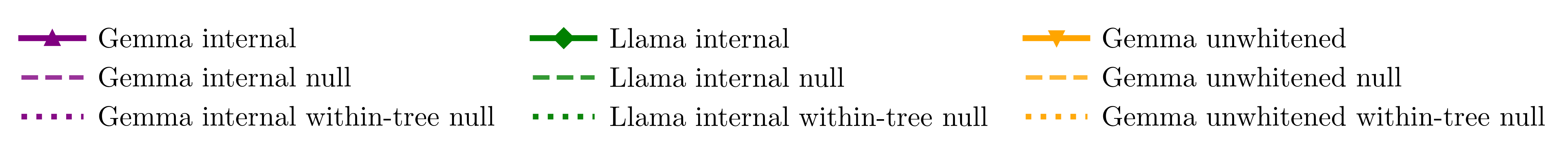}
    \caption{
    \textbf{Centered unwhitened and internal-activation representation controls.}
    We evaluate alignment using globally centered unwhitened Gemma embeddings, globally centered Gemma middle-layer residual-stream activations, and globally centered Llama middle-layer residual-stream activations.
    For each representation, the global mean vector is subtracted before computing the eigenspace alignment.
    Centering removes the constant anisotropic component that can otherwise increase the within-tree shuffle baseline.
    The centered representations remain substantially above both the global shuffle and within-tree shuffle baselines, showing that the hierarchical alignment is not explained by a shared global mode.
    }
    \label{fig:centered_internal_unwhitened_controls}
\end{figure*}

Together, these robustness checks support three conclusions. First, the alignment signal is not dependent on the specific decay family used to define the theoretical kernel. Second, the effect persists when the tree-completion constraint is relaxed from \(L=3\) to \(L=2\), yielding many more eligible roots. Third, the effect is visible not only in the main whitened embedding representation but also in centered embeddings and even in internal activations from both Gemma and Llama models. The baselines therefore rule out both vocabulary-level label randomization and within-tree vector-content effects as sufficient explanations for the observed alignment.

\clearpage
\section{Concept-vector diagnostics of hierarchical geometry}
\label[app]{app:park_consistency}

The main text shows that the fitted co-occurrence model reproduces the parent--child
orthogonal-innovation diagnostic of Park et al.~\cite{park2025geometrycategoricalhierarchicalconcepts}.
Here we give the details of the concept-vector estimator and report additional diagnostics.

Park et al. study hierarchical geometry at the level of \emph{concept vectors}. For each
WordNet synset \(w\), let \(Y(w)\) denote the set of tokens corresponding to \(w\) or one of
its descendants. The intended meaning of the concept vector \(\bar\ell_w\) is membership in
this descendant set. Their functional postulate is that concept directions should support
targeted interventions: changing from \emph{bird} to \emph{fish} should change which animal is
represented while preserving the broader \emph{animal} component.

For a hierarchy, this predicts an ancestor--descendant decomposition. If \(p\) is the parent of
\(w\), then the child concept vector should decompose as
\[
    \bar\ell_w \approx \bar\ell_p + \delta_w,
    \qquad
    \delta_w := \bar\ell_w-\bar\ell_p,
\]
where the innovation \(\delta_w\) is approximately orthogonal to the parent vector
\(\bar\ell_p\). The diagnostic plotted in \cref{fig:Parkconsistency} is therefore
\[
    \cos(\bar\ell_w-\bar\ell_p,\bar\ell_p),
\]
which should concentrate near zero for true parent--child pairs.

\paragraph{Concept-vector estimator.}

Let \(g(y)\in\mathbb R^d\) denote the representation vector associated with token \(y\).
In the LLM experiments, \(g(y)\) is the centered and whitened Gemma unembedding vector. In the
synthetic experiments, \(g(y)\) is the corresponding theoretical co-occurrence embedding
vector. For each concept \(w\), we estimate the concept vector from a training subset
\(Y_{\mathrm{tr}}(w)\subset Y(w)\), formed by independently selecting \(70\%\) of descendant
tokens for each synset, following Park et al.

Define
\[
    \mu_w
    :=
    \mathbb E[g(y)\mid y\in Y_{\mathrm{tr}}(w)],
    \qquad
    \Sigma_w
    :=
    \operatorname{Cov}(g(y)\mid y\in Y_{\mathrm{tr}}(w)).
\]
Operationally, the estimator chooses a direction whose projected target-set mean is large
relative to its projected target-set variance:
\[
    \max_{u\in\mathbb R^d}
    \frac{(u^\top \mu_w)^2}{u^\top \Sigma_w u}.
\]
The maximizing direction is
\[
    \tilde g_w
    =
    \frac{\Sigma_w^\dagger \mu_w}
    {\|\Sigma_w^\dagger \mu_w\|_2},
\]
where \(\Sigma_w^\dagger\) denotes the Moore--Penrose pseudoinverse. The corresponding
concept vector is
\[
    \bar\ell_w
    =
    (\tilde g_w^\top \mu_w)\tilde g_w.
\]
In the isotropic case \(\Sigma_w\propto I\), this reduces to \(\bar\ell_w=\mu_w\).

Since the representation dimension is larger than the number of descendant samples for many
concepts, the empirical covariance is poorly conditioned. We therefore estimate
\(\Sigma_w\) using the Ledoit--Wolf shrinkage estimator~\cite{ledoitwolf}
before applying the pseudoinverse. This covariance regularization is used only for estimating
the Park et al. concept vectors; it is distinct from the full-vocabulary whitening used for
Gemma unembeddings in \cref{app:gemma_unembeddings}.

\paragraph{Synthetic co-occurrence embeddings.}

To test whether the same diagnostic can arise without an LLM-specific hierarchical mechanism,
we repeat the concept-vector estimation procedure on theoretical co-occurrence embeddings. We
construct a Gram matrix from the fitted distance kernel and take its top \(d=2048\) positive
eigenmodes, as in \cref{eq:w2v_form}. Applying the same estimator to these embeddings yields
the same qualitative parent--child orthogonality pattern as Gemma unembeddings: true
parent--child innovations concentrate near zero, while shuffled-parent baselines are displaced.

\paragraph{Additional diagnostics.}

\Cref{fig:Parkconsistencyapp} reports additional diagnostics following Park et al. The top row
compares projection-based concept separation: points in a concept's held-out descendant set
project more strongly onto the estimated concept direction than randomly selected or shuffled
tokens. The bottom row evaluates orthogonal innovations along hypernym chains and compares
against standard null controls, including shuffled parent assignments and globally shuffled
embedding vectors. The synthetic co-occurrence embeddings reproduce the qualitative separation
between true hierarchy relations and null baselines, although with larger variance than Gemma.

This comparison changes the interpretation of the orthogonal-innovation pattern. Park et al.'s
framework postulates ideal concept vectors under an appropriate inner product, for which
hierarchy implies exact orthogonality; empirical deviations are then naturally treated as
noise or imperfect realization of the ideal. Our results show that the same concept-vector
construction yields approximate orthogonality in a purely co-occurrence-driven model. Thus,
the concept-level hierarchical geometry observed in LLM unembeddings need not by itself be
evidence for a hierarchy-specific functional mechanism; it can arise from the same spectral
structure that produces hierarchical splitting geometry.

\begin{figure*}[t]
    \centering
    \includegraphics[width=0.95\textwidth]{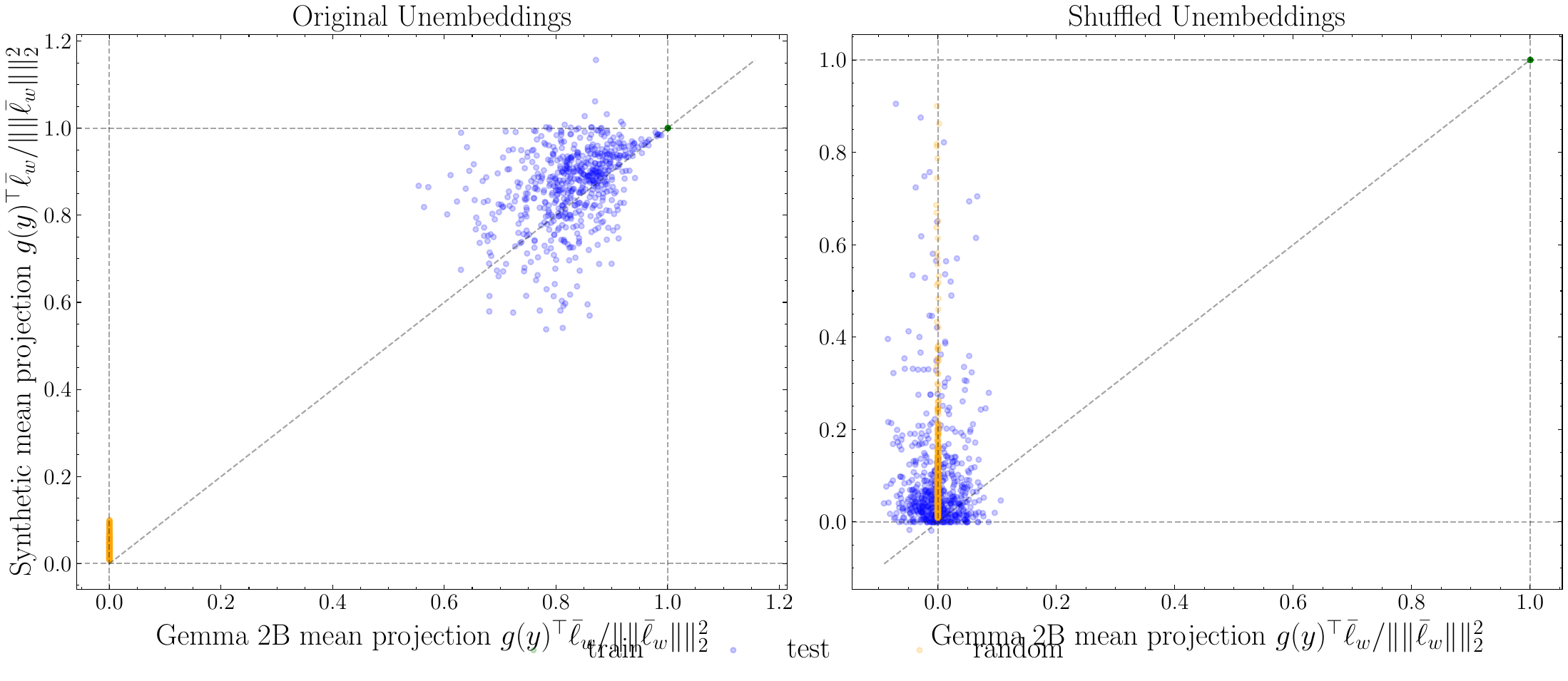}
    \includegraphics[width=0.95\textwidth]{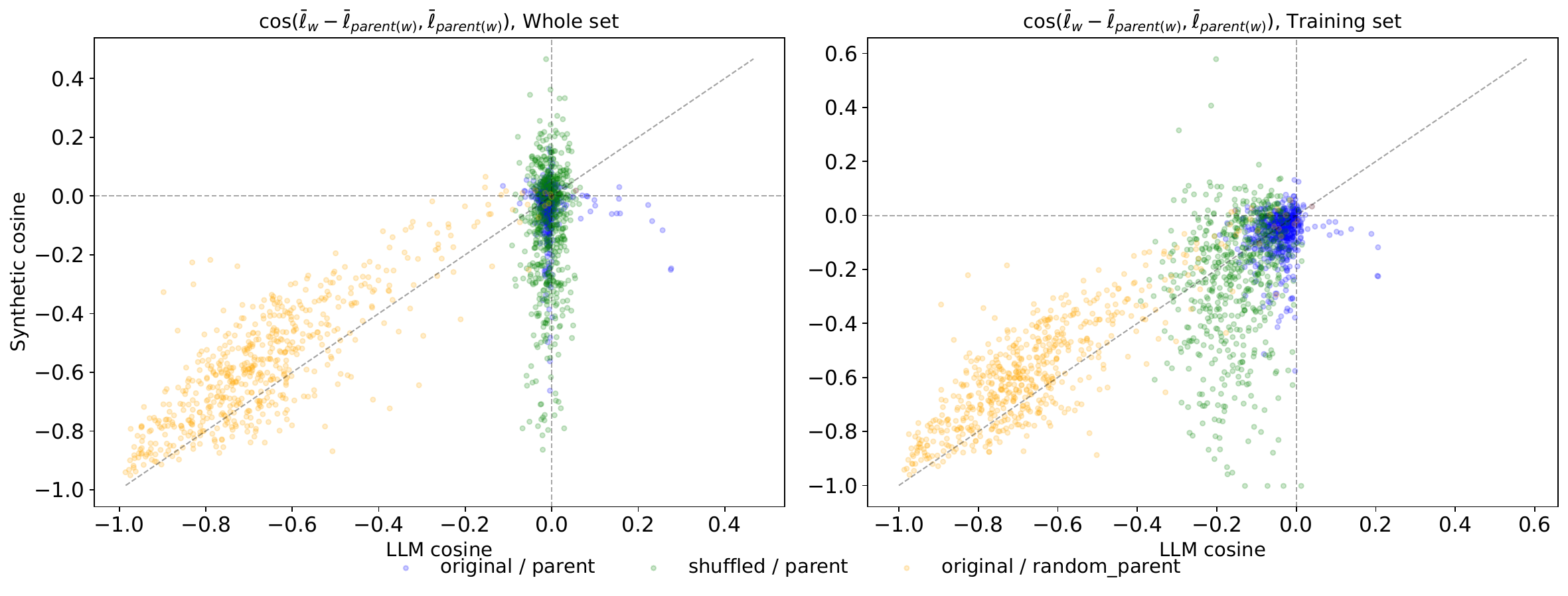}
    \caption{
    \textbf{Additional concept-vector diagnostics for Gemma and co-occurrence embeddings.}
    We apply the concept-vector estimator of Park et al.~\cite{park2025geometrycategoricalhierarchicalconcepts}
    to Gemma unembeddings and to theoretical co-occurrence embeddings constructed from the fitted
    distance kernel. Concept vectors are estimated from independently selected \(70\%\) training
    subsets of descendant tokens for each synset.
    \emph{Top row:} projection-based concept separation. Estimated concept directions assign larger
    projections to held-out descendant tokens than to random or shuffled-token baselines.
    \emph{Bottom row:} orthogonal-innovation diagnostics along hypernym chains. True hierarchy
    relations show near-orthogonal innovations, while shuffled-parent and globally shuffled-vector
    baselines are displaced. The co-occurrence model reproduces the qualitative structure of the
    Gemma diagnostics, with larger variance.
    }
    \label{fig:Parkconsistencyapp}
\end{figure*}

\end{document}